\documentclass[11pt]{article}

\usepackage[margin=1.1in]{geometry}
\usepackage{latexsym}
\usepackage{amsmath,amssymb,amsfonts,amsthm}
\usepackage{mathtools}
\usepackage{algorithm}
\usepackage[noend]{algpseudocode}
\usepackage{xcolor}
\usepackage[authoryear, round]{natbib}
\usepackage{graphicx}
\usepackage{parskip}
\usepackage{booktabs,array,tabularx}
\usepackage{hyperref}
\usepackage{cleveref}
\usepackage{enumitem}
\usepackage{bm}
\usepackage{titlesec}
\usepackage{needspace}
\usepackage{float}
\usepackage[section]{placeins}
\usepackage{fancyhdr}
\usepackage{microtype}
\usepackage{caption}
\usepackage{tikz}
\usetikzlibrary{arrows.meta,positioning,fit,backgrounds}

\newtheorem{theorem}{Theorem}
\newtheorem{proposition}{Proposition}
\newtheorem{corollary}{Corollary}
\newtheorem{lemma}{Lemma}

\theoremstyle{remark}

\newtheorem{assumption}{Assumption}
\newtheorem{example}{Example}

\hypersetup{
  pdftitle={Conformal Selection under Prevalence Shift: FDR Bounds and Finite Calibration},
  pdfauthor={Seungjin Choi},
  colorlinks=true,
  linkcolor=blue!70!black,
  citecolor=blue!50!black,
  urlcolor=blue!60!black
}

\providecommand{\E}{\mathbb{E}}
\providecommand{\Prob}{\mathbb{P}}

\newcommand{\Ical}{\mathcal{I}}

\newcommand{\Ccal}{\mathcal{C}}

\newcommand{\Scal}{\mathcal{S}}

\newcommand{\FDR}{\mathrm{FDR}}

\newcommand{\FDP}{\mathrm{FDP}}
\newcommand{\Power}{\mathrm{Power}}

\newcommand{\ind}{\mathbb{I}}

\titleformat{\section}{\large\bfseries\color{blue!60!black}}{{\thesection}}{1em}{}[\titlerule]
\titleformat{\subsection}{\normalsize\bfseries\color{blue!40!black}}{{\thesubsection}}{1em}{}
\titleformat{\subsubsection}{\normalsize\itshape\bfseries}{{\thesubsubsection}}{1em}{}

\begin{document}

\begin{center}
  {\LARGE\bfseries \textsf{Conformal Selection under Prevalence Shift: \\ [10pt]
  FDR Bounds and Finite Calibration}}\\[2em]
  {\large  Seungjin Choi}\\[1em]
  {\normalsize CROID Research and aSSIST University, Korea}
\end{center}

\vspace{0.5em}
\noindent\rule{\linewidth}{1.5pt}
\vspace{0.5em}

\begin{abstract}
Conformalized Benjamini--Hochberg (cfBH) uses labeled calibration data to select
candidates while controlling the false discovery rate (FDR) under suitable assumptions.
We examine how differences between source and test populations affect FDR and power.
When non-target scores have the same distribution in both populations, existing results
guarantee FDR control for clipped cfBH if non-targets are at least as common in the source
population as in the test population. When this prevalence condition fails,
FDR can exceed the nominal level.
We determine the worst-case FDR in large samples. For example, a nominal level of
10\% can yield FDR approaching 40\% when the source population has 20\% non-targets
and the test population has 80\%.
Calibrating only against labeled non-targets avoids this restriction on their proportions,
provided their score distributions still agree. However, every nonempty selection must
contain a minimum number of candidates. This requirement can limit power when few
true targets are present. More calibration non-targets can lower the minimum.
We use the same minimum to bound FDR inflation when non-target score distributions
differ. In worst-case examples, a contamination fraction inversely proportional to the
number of calibration non-targets can drive FDR close to 100\% in large test pools.
A random labeled sample from the test population provides a representative calibration
reference. Pooling its non-targets with source non-targets can improve power when their
score distributions agree. Without this agreement, we give a correction for pooling
that controls FDR but selects only a subset of the candidates selected using the random
sample alone. Simulations and molecular-screening experiments on AqSolDB, with no
shared molecular identities across data splits, illustrate these findings.
\end{abstract}

\section{Introduction}
\label{sec:intro}

In molecular screening, predictive models can rank thousands of compounds
for costly follow-up experiments \citep{VamathevanJ2019nrdd,StokesJ2020cell}.
The next step is to choose
which compounds to test in the laboratory.  A good ranking alone is not
enough. Among the selected compounds, some will fail verification.  We would
like to find as many promising compounds as possible while controlling the
\emph{false discovery rate} (FDR), the expected proportion of non-targets
among the selected candidates.  Beyond molecular screening, similar prediction-guided decisions arise in
outlier detection and other large-scale screening tasks
\citep{JinY2023jmlr,BatesS2023aos}.

Conformalized Benjamini--Hochberg (cfBH) addresses this
problem \citep{JinY2023jmlr,BatesS2023aos}.  It compares a candidate's
prediction score with scores from labeled calibration data, constructs
conformal p-values, and applies the Benjamini--Hochberg procedure \citep{BenjaminiY95jrsssb} to select
candidates.  Its guarantees depend on how the calibration data relate to
the population being screened.  In practice, those populations need not
have the same distribution.  For example, a historical molecular dataset
may have been enriched with successful compounds, while the new screening
pool contains mostly unsuccessful ones.  Molecular-property benchmarks
illustrate the range of labeled datasets available for score fitting
\citep{WuZ2018chemsci,SorkunM2019sd}.

Two differences are particularly relevant.  First, the \emph{prevalence}
of desirable candidates may change.  Second, even among non-targets, the
distribution of prediction scores may change.  The latter matters because
a non-target receiving an unusually high score can be mistaken for a
promising discovery.  These changes raise the question of \emph{when conformal BH can
safely use labeled source data to select candidates from a different
population and how distribution shifts affect FDR control
and selection power.}

To formalize the prevalence question, we call a desirable candidate a
\emph{target} and any other candidate a \emph{null}.  Let $p$ and $r$
denote the proportions of nulls in the labeled source population and
the test population, respectively.  The cfBH framework already provides
a null-reference FDR guarantee and a normalization identity linking
null-reference ranks to clipped cfBH \citep{JinY2023jmlr}.  Combining these results with a binomial calculation gives a
finite-sample FDR bound under prevalence shift when source and test
nulls have the same score distribution. An immediate consequence is
that clipped cfBH controls FDR at level $q$ whenever $p\ge r$.
We show that its worst-case limiting FDR is
$r\min\{1,q/p\}$ (Corollary~\ref{cor:shift-sharp}).  For example, if $q=0.10$, $p=0.20$, and
$r=0.80$, the worst-case FDR can approach $40\%$, although the nominal
level is $10\%$.

Null-reference cfBH (cfBH0) takes a different approach from clipped cfBH.
It calibrates test candidates using only labeled source nulls, without
adjusting for the source class proportion. When source and test null
scores have the same distribution, this procedure controls FDR
regardless of differences in target prevalence. However, its selection
power can be limited by the number of calibration nulls. With only
$k$ nulls, the smallest attainable conformal p-value is $1/(k+1)$,
which may prevent BH from selecting promising candidates even when
the prediction scores are informative. We quantify how this limited
p-value resolution constrains selection power.

This limitation motivates the use of additional labeled data from the
test population. We consider a small random sample of test-population
candidates whose labels are verified before screening the remaining
candidates. This \emph{pilot sample} provides additional null scores
and information about the test null proportion. When source and test
null scores have the same distribution, combining source and pilot
nulls can provide finer p-value resolution and improve selection power.
However, when their distributions differ, pooling may inflate FDR.
Calibrating with the pilot alone retains FDR control, although a
small pilot may limit power. We establish finite-sample bounds that
quantify the effect of null-score mismatch and the risks of pooling
unrepresentative source nulls. For arbitrary source nulls, we derive a
level correction that controls FDR. Its selections are always a subset
of the pilot-only selections,
so this correction cannot improve power over using the pilot alone.

\Needspace{5\baselineskip}
For a fixed scoring rule, our results explain how differences between the
source and test populations affect FDR, how limited calibration data
restrict power, and what can be gained by combining source data with
representative data from the test population. Our main contributions are summarized below.

\begin{enumerate}[label=(\roman*),leftmargin=2.2em,itemsep=2pt,topsep=2pt]

\item \textbf{Changes in target prevalence can inflate FDR.}
Even when non-target scores have the same distribution across populations,
clipped cfBH can lose FDR control if non-targets become more common in the
test population. Building on the results of \citet{JinY2023jmlr}, we obtain
a finite-sample FDR bound (Proposition~\ref{prop:shift-bh}).
We also determine the largest possible FDR in large samples, showing that
the limiting bound cannot be improved (Corollary~\ref{cor:shift-sharp}).

\item \textbf{Limited calibration data can restrict power.}
When calibration uses only labeled non-targets, BH must select either
no candidates or at least a certain number. More calibration non-targets
can reduce this minimum. If the test pool contains fewer true targets
than the minimum, every nonempty selection necessarily includes false
discoveries. FDR control then bounds the probability of a nonempty
selection and the expected fraction of targets selected.
Proposition~\ref{prop:resolution} gives a power bound for every
fixed score and test label configuration, refining the rank-based
argument of \citet{MaryD2022ejs}.

\item \textbf{Changes in non-target scores can cause severe FDR failures.}
The same minimum selection size helps bound how much FDR can increase
when calibration and test non-targets have different score distributions
(Theorem~\ref{thm:score-null-shift}).
We also construct examples in which a contamination fraction inversely
proportional to the number of calibration non-targets drives FDR toward
100\% as the test pool grows (Proposition~\ref{prop:shift-breakdown}).

\item \textbf{Corrected pooling cannot improve on the pilot alone.}
Pooling source non-targets with representative pilot non-targets can
weaken the FDR guarantee if their score distributions differ. We derive
an explicit inflation factor, attained in examples with just one test
non-target (Proposition~\ref{prop:unmatched-pooling}).
Reducing the testing level by this factor restores FDR control.
However, the corrected method never selects a candidate that would be
excluded when using the pilot alone
(Corollary~\ref{cor:pool-containment}).

\end{enumerate}
Controlled simulations isolate these effects. Molecular-screening
experiments on AqSolDB, with no shared molecular identities across
data splits, illustrate their practical implications.

Section~\ref{sec:related} reviews related work, and
Section~\ref{sec:setup} introduces the calibration rules.
The next sections analyze prevalence shift and finite calibration,
random pilot labels, and null-score distribution mismatch.
Section~\ref{sec:experiments} presents the experimental methods and
results. Proofs are given in the appendices.

\section{Related Work}
\label{sec:related}

\subsection{Prediction-guided molecular screening}

Machine learning helps prioritize compounds for costly experimental
follow-up \citep{VamathevanJ2019nrdd}. For example,
\citet{StokesJ2020cell} used model predictions to guide antibiotic
screening. MoleculeNet \citep{WuZ2018chemsci} and AqSolDB
\citep{SorkunM2019sd} provide molecular datasets for predictive modeling.
Accurate rankings alone, however, do not guarantee FDR control among
selected compounds. Our AqSolDB experiment evaluates screening
procedures retrospectively using existing measurements.

\subsection{Conformal selection and FDR control}

The Benjamini--Hochberg procedure \citep{BenjaminiY95jrsssb} controls FDR
under suitable validity and dependence conditions on the p-values.
Conformal methods obtain valid p-values through exchangeability
\citep{VovkV2005book,AngelopoulosAN2023ftml}, including calibration
within label classes \citep{VovkV2012acml}.
\citet{BatesS2023aos} establish FDR guarantees for conformal outlier
p-values, including an adaptive procedure based on the Storey
correction \citep{StoreyJD2002jrsssb}.

\citet{JinY2023jmlr} introduce conformal selection by prediction.
They provide a conditional FDR guarantee for cfBH0 and show how its
p-values relate to those of clipped cfBH.
When source and test null scores have the same distribution, these
results and a binomial calculation give FDR control for clipped cfBH
whenever nulls are at least as common in the source population as in
the test population.
We show that the resulting limiting FDR bound is sharp, even for
two-valued scores.

\citet{MaryD2022ejs} compare empirical BH
with oracle BH and show how a small calibration sample can limit power.
They also bound the probability of any true discovery by
combining the minimum attainable p-value with FDR control.
Our conditional power bound refines this argument by retaining
the exact minimum selection size and the test null fraction.
Section~\ref{sec:resolution} explains this connection.

Further work develops score optimization, improved data efficiency,
and adaptive novelty detection
\citep{BaiT2024arxiv,QinJ2025arxiv,HuoY2026jrsssb,MarandonA2024aos}.
\citet{GaoZ2026arxiv} consider boundary FDR, the probability that
the least extreme selected candidate is null. This criterion differs
from FDR, which averages the proportion of nulls among all selected
candidates.

\subsection{Distribution shift and imperfect reference data}

Binary label shift changes class proportions while preserving both
class-conditional distributions \citep{LiptonZ2018icml}.
Our common-null assumption preserves the null distribution but allows
the distribution of true targets to differ between populations.
\citet{JinY2023jmlr} also discuss this distinction.
For a fixed score, our rank arguments require only matching source
and test null-score distributions.

Conformal Bayes methods combine Bayesian predictive adaptation with
importance-weighted calibration under label shift
\citep{Choi2026eiml,LeeHS2026eiml}. These methods concern predictive
coverage and prediction-set efficiency, while our analysis concerns
FDR and power for selection with a fixed score.
Weighted conformal selection addresses covariate shift and inlier
shift through a different multiple-testing construction
\citep{JinY2026biometrika}.
Other related work examines conformal prediction beyond
exchangeability \citep{BarberRF2023aos} and outlier detection with
contaminated reference data \citep{BashariM2025icml}.
We quantify how prevalence mismatch, limited calibration data,
and changes in null-score distributions affect cfBH procedures,
including the benefits and limits of a representative labeled pilot.

\Needspace{7\baselineskip}
\section{Conformal Selection and Calibration Rules}
\label{sec:setup}

\subsection{Targets, nulls, and the screening score}

For a response $Y$ and a prespecified target region $\Ical$, define
$Z=\ind\{Y\in\Ical\}$. A candidate is a \emph{target} if $Z=1$
and a \emph{null} otherwise. Let $g(x)$ be a prediction score, with
larger values indicating more promising candidates. The score may
be obtained from a classifier, regressor, or other predictive model.
For our theoretical guarantees, we assume that $g$ is trained on
data independent of the calibration and test samples and then held
fixed. The guarantees hold conditional on the training data.

We distinguish two populations.  There are $n$ labeled source-calibration
observations and $m\ge1$ unlabeled test candidates. Their population
null proportions are
\[
  p=\Prob_{\rm source}(Z=0),\qquad
  r=\Prob_{\rm test}(Z=0).
\]
Write $K$ for the number of nulls among the $n$ source observations.  The
selection goal is to keep the expected false discovery proportion at most
$q\in(0,1)$, while selecting as many true targets as possible.
\[
 \FDP=\frac{\#\{\text{selected nulls}\}}{1\vee\#\{\text{selected}\}},
 \qquad \FDR=\E[\FDP],\qquad
 \Power=\E\!\left[\frac{\#\{\text{selected targets}\}}
 {1\vee\#\{\text{test targets}\}}\right].
\]
The test labels remain hidden when selections are made.

\begin{assumption}[Independent samples with a common null distribution]
\label{ass:shift}
Source calibration pairs are independent and identically distributed,
with null probability $p\in(0,1]$. The $m$ test pairs are i.i.d.,
with null probability $r\in(0,1)$, and independent of the source
sample. Conditional on $Z=0$, the covariates $X$ have the same
distribution $P_0$ in both populations. Target distributions,
when present, may differ.
\end{assumption}

This assumption allows the proportion of nulls to change between
populations while preserving the distribution of covariates among
nulls. Because the prediction score $g$ is fixed, it also implies
that source and test null scores have the same distribution.
For a fixed $g$, the rank arguments use only this equality of null-score
distributions. Section~\ref{sec:null-robustness} allows these distributions to differ.
Binary label shift is a special case in which the feature
distributions of both targets and nulls remain unchanged,
while their proportions may differ
\citep{SaerensM2002neco,LiptonZ2018icml,AzizzadenesheliK2019iclr,GargS2020neurips}.
In contrast, label shift in a continuous response $Y$ need not preserve
the null distribution after grouping responses into $Z$.
See Appendix~\ref{app:label-grouping}.

\subsection{Null-reference p-values and BH}

Let $\Ccal_0=\{i:Z_i=0\}$ be the source null indices, so $K=|\Ccal_0|$
and $K\sim\mathrm{Bin}(n,p)$.  For test candidate $j$, define the
\emph{null-reference p-value}
\begin{equation}\label{eq:null_pvalue}
 p_j^{\mathrm{null}}
 =\frac{1+\#\{i\in\Ccal_0:g(X_i)\ge g(X_j^{\rm t})\}}{K+1}.
\end{equation}
It is small when a candidate scores higher than nearly all labeled
source nulls.  The extra one counts the candidate itself to prevent a
zero p-value.  Scores tied with the candidate count against selection.
When $K=0$, every p-value is one.

The null-reference cfBH0 procedure applies the
\emph{Benjamini--Hochberg (BH)} method
\citep{BenjaminiY95jrsssb} to these p-values at the
nominal FDR level $q$. Specifically, BH sorts the $m$
p-values and finds the largest $R$ satisfying
$p_{(R)}\le qR/m$. It then selects the $R$ candidates
with the smallest p-values, or none if no such $R$
exists. Other calibration rules considered below use
the same BH procedure but may adjust the testing level
from $q$ to a different value $\alpha$.

Our FDR analysis builds on the conditional guarantee for null-reference BH
in \citet[Proposition~8]{JinY2023jmlr}, with related results
in \citet{BatesS2023aos}. The following lemma states this guarantee
at any fixed testing level $\alpha$, using conservative ranks to
accommodate tied prediction scores. Appendix~\ref{app:ties} provides
a direct proof of this version.
Throughout, conditional statements about membership labels concern
configurations with positive probability.

\begin{lemma}[Conditional FDR control for null-reference BH]
\label{lem:conditional-bh}
Under Assumption~\ref{ass:shift}, for every fixed configuration of the
source and test binary membership labels, BH applied to
\eqref{eq:null_pvalue} at a fixed level $\alpha\in(0,1]$ satisfies
\[
 \E[\FDP\mid\text{source and test membership labels}]
 \le \alpha\frac{m_0}{m},
\]
where $m$ is the total number of test candidates and $m_0$ is the
number of null test candidates. The expectation averages over the
calibration non-target scores and test scores while holding the
membership labels fixed. Ties in these scores are allowed.
\end{lemma}

For cfBH0, setting $\alpha=q$ and averaging over the membership
labels gives $\FDR\le qr\le q$. The guarantee averages over
the calibration and test scores. It need not hold when the
calibration scores are also fixed and only the test scores remain
random. Marginal validity of individual p-values alone does not
ensure FDR control for BH under arbitrary dependence
\citep{BenjaminiY2001aos}.

\subsection{Four conformal BH calibration rules}
\label{sec:four-rules}

The procedures differ only in which nulls form the reference and in the BH level.
The clipped and null-reference p-values are related by
an exact identity noted by \citet[Remark~9]{JinY2023jmlr}.
\begin{equation}\label{eq:shift-rank-identity}
 p_j^{\mathrm{clip}}
 =\frac{1+\#\{i\in\Ccal_0:g(X_i)\ge g(X_j^{\rm t})\}}{n+1}
 =\frac{K+1}{n+1}\,p_j^{\mathrm{null}}.
\end{equation}
Thus, clipped cfBH at level $q$ is equivalent to
null-reference BH at the adjusted level
$\alpha_K=\min\{1,q(n+1)/(K+1)\}$.

Suppose an independent random pilot of $b$ labeled observations from
the test population is available, of which $B$ are nulls. The pilot
provides two types of information. Its $B$ null scores can be used
for calibration, while its $b$ membership labels provide information
about the null proportion in the test population.

To construct the p-values, we can use the $B$ pilot nulls alone.
If source and test null scores have the same distribution, we can
instead pool them with the $K$ source nulls. The null-reference
p-value denominators are therefore $B+1$ and $K+B+1$, respectively.
Pooling gives a larger calibration reference and finer p-value resolution.

To choose the BH level, both procedures use the same information
about the test null proportion, namely the $B$ nulls among the $b$
random pilot observations. We therefore use the common level
\begin{equation}\label{eq:pilot-level}
 \alpha(B)=\min\left\{1,\frac{q(b+1)}{B+1}\right\}.
\end{equation}
The $K$ source nulls supply additional null scores, but no additional
random membership labels from the test population. They therefore
enter the pooled p-value denominator without changing this level
adjustment. Using the same level is justified by
Lemma~\ref{lem:conditional-bh}, which gives the same conditional FDR
bound for either valid calibration reference, regardless of its size.

Algorithm~\ref{alg:four-rules} summarizes four calibration rules.
They are \emph{source null-reference BH (cfBH0)},
\emph{source clipped cfBH}, \emph{pooled pilot add-one}, and
\emph{pilot-only clipped cfBH}.
The first two use only source calibration labels, whereas the
latter two incorporate labels from a random test-population pilot.
We analyze how their FDR and power depend on the calibration data.

\begin{algorithm}[!hbp]
\caption{Four conformal BH calibration choices for a frozen score.
Their validity conditions are summarized in Table~\ref{tab:decision-guide}.}
\label{alg:four-rules}
\begin{algorithmic}[1]
\Require Frozen score $g$, labeled source of size $n$, test candidate scores,
nominal FDR level $q$, and optionally an independent labeled pilot of
size $b$ from the test population, separate from the test pool
\State Find $K$ source nulls.  Compute $p_j^{\rm null}$ from
\eqref{eq:null_pvalue}.
\State \textbf{Source null-reference BH}\quad apply BH to $p_j^{\rm null}$ at $q$.
\State \textbf{Source clipped cfBH}\quad apply BH to
$p_j^{\rm clip}=\frac{K+1}{n+1}p_j^{\rm null}$ at $q$.
\If{a pilot of $b$ labels is available}
  \State Let $B$ be its number of nulls and compute $\alpha(B)$ from
  \eqref{eq:pilot-level}.
  \State \textbf{Pooled pilot add-one}\quad compute null-reference p-values against the
  $K+B$ source-plus-pilot nulls and apply BH at $\alpha(B)$.
  \State \textbf{Pilot-only clipped cfBH}\quad compute null-reference p-values against the
  $B$ pilot nulls alone and apply BH at $\alpha(B)$.
\EndIf
\end{algorithmic}
\end{algorithm}

\begin{example}[P-value resolution with a 100-label pilot]
\label{ex:running}
Suppose $n=999$ and the labeled source contains $K=199$ nulls.
At $q=0.10$, source clipping runs null-reference BH at level
$0.10(1000)/200=0.50$, five times the ordinary level.  If a random pilot
has $b=100$ labels, including $B=80$ nulls, its adjustment is
$\alpha(B)=0.10(101)/81\approx0.125$.
The pooled reference has $279$ nulls and a minimum p-value of
$1/280$. Pilot-only calibration has $80$ nulls and a minimum p-value
of $1/81$. Both procedures use the same BH level. Pooling offers
finer p-value resolution, which can improve power when source and
test null-score distributions agree. The actual gain depends on
the score distributions.
\end{example}

The four rules and the assumptions behind their guarantees are summarized
in Table~\ref{tab:decision-guide}. These guarantees apply to each
specified procedure. They do not justify choosing a procedure after
comparing its selections with those of the other rules.

\begin{table}[ht!]
\centering\small
\caption{Summary of four calibration rules.  ``Same nulls''
means that the source and test null-score distributions agree. The main
common-null assumption uses the stronger condition of matching null
covariate distributions. All scores are fitted independently of the
calibration and test data. The pilot-based rules use an independent
random pilot from the test population. Appendix~\ref{app:finite-library}
gives the corresponding conditions for a finite library.}
\label{tab:decision-guide}
\begin{tabularx}{\linewidth}{@{}>{\raggedright\arraybackslash}p{.22\linewidth}>{\raggedright\arraybackslash}X >{\raggedright\arraybackslash}p{.30\linewidth}@{}}
\toprule
Rule & How it uses labels & What the guarantee needs \\
\midrule
Source clipped cfBH & Uses all source calibration labels to scale null
p-values & Same nulls. Controls FDR at $q$ if $p\ge r$ or $r\le q$ \\
Null-reference BH (cfBH0) & Ranks candidates against source \emph{nulls}
only & Same nulls. Limited calibration can restrict power \\
Pooled pilot add-one & Adds pilot nulls to source nulls and
adjusts the BH level from pilot labels & Same nulls. More references but
still relies on the source \\
Pilot-only clipped cfBH & Uses only a random representative test-population pilot &
Pilot and test samples represent the same population. Source may differ \\
\bottomrule
\end{tabularx}
\end{table}

\section{Prevalence Shift and Finite Calibration}
\subsection{FDR control under prevalence shift}
\label{sec:shift}

We first consider prevalence shift under a common null distribution.
The conditional distribution of nulls remains unchanged between
the source and test populations, while their proportions $p$ and $r$ may differ.
Clipped cfBH multiplies the null-reference p-values by
$(K+1)/(n+1)\le1$.
Consequently, applying clipped cfBH at the nominal level $q$
is equivalent to applying null-reference BH at the adjusted
level
\[
\alpha_K=\min\left\{1,\frac{q(n+1)}{K+1}\right\}.
\]
When the source calibration sample contains fewer nulls,
this adjusted level can increase, making selection more aggressive
even if the test population contains mostly nulls.

The conditional null-reference BH guarantee and the
clipped-to-null p-value identity of \citet{JinY2023jmlr}
lead directly to a finite-sample FDR bound. Although the adjusted BH
level $\alpha_K$ is random, it is fixed when we condition on the
source membership labels as in Lemma~\ref{lem:conditional-bh}.
Averaging over $K\sim\mathrm{Bin}(n,p)$ then gives the following
proposition for any fixed score, including scores with ties.

\begin{proposition}[FDR bounds for clipped cfBH under prevalence shift]
\label{prop:shift-bh}
Under Assumption~\ref{ass:shift}, for every fixed measurable real-valued score
$g$, including scores with ties, conservative null-reference BH and clipped cfBH satisfy
\begin{align}
 \FDR_{\mathrm{null}}&\le qr,\label{eq:shift-nccs}\\
 \FDR_{\mathrm{clip}}&\le
 r\E\left[\min\left\{1,\frac{q(n+1)}{K+1}\right\}\right]
 \le\min\left\{r,\frac{qr}{p}\bigl[1-(1-p)^{n+1}\bigr]\right\}.
 \label{eq:shift-bh-bound}
\end{align}
For fixed $p,r,q$ and any sequence of test sample sizes
$m=m_n\ge 1$,
\begin{equation}
 \label{eq:shift-bh-limit}
 \limsup_{n\to\infty}\FDR_{\mathrm{clip}}
 \le r\min\{1,q/p\}.
\end{equation}
If $p<q$, clipped cfBH selects the entire test pool with
probability tending to one as $n\to\infty$, and
$\FDR_{\mathrm{clip}}\to r$.
\end{proposition}

The key calculation, using $K\sim\mathrm{Bin}(n,p)$, is
\[
 \E\left[\frac1{K+1}\right]
 =\frac{1-(1-p)^{n+1}}{(n+1)p}.
\]
Clipped cfBH is equivalent to null-reference BH at level $\alpha_K$.
Conditional on the source and test membership labels, $\alpha_K$ is
fixed, so Lemma~\ref{lem:conditional-bh} bounds the conditional expected
FDP by $\alpha_Km_0/m$. Averaging over these labels gives the stated
FDR bound. Appendix~\ref{app:prevalence-details} provides the full proof.

When $p<q$, the probability that $\alpha_K=1$ tends to one as
$n\to\infty$. On this event, clipped cfBH selects the entire test pool,
regardless of the prediction scores. More generally, the stated bounds
are upper bounds, so the actual FDR may be lower for a given score and
pair of populations. Corollary~\ref{cor:shift-sharp} below shows that
the limiting bound cannot be improved in general.

For a fixed score, these upper bounds remain valid when the score
distributions among true targets differ, provided the source and test
null-score distributions agree. Section~\ref{sec:null-robustness}
examines what happens when the null-score distributions also differ.

\begin{corollary}[Finite-sample FDR control]
\label{cor:safe}
Under Assumption~\ref{ass:shift}, conservative clipped cfBH
controls FDR at level $q$ whenever $p\ge r$ or $r\le q$.
This holds for every fixed score, including scores with ties.
\end{corollary}

Both conditions follow from Proposition~\ref{prop:shift-bh}.
When $p\ge r$, its finite-sample bound is at most $q$.
When $r\le q$, control follows because the bound is at most $r$.
The limiting upper bound exceeds $q$ precisely when
$r>\max\{p,q\}$. For example, with $q=0.10$, $p=0.20$, and
$r=0.80$, this bound is $0.40$, four times the nominal level.
The construction below shows that FDR can approach this bound,
although it need not do so for every score.
Figure~\ref{fig:prevalence-envelope} compares the limiting clipped-cfBH
bound with the finite-sample null-reference BH bound as $p$ varies.

\begin{figure}[!htbp]
\centering
\begin{tikzpicture}[x=10cm,y=4.2cm,>=Latex,font=\small]
 \fill[blue!3] (.02,0) rectangle (.8,.84);
 \draw[->] (0,0) -- (1.045,0);
 \node at (.5,-.18) {Source null proportion $p$};
 \draw[->] (0,0) -- (0,.89) node[above right] {FDR bound};
 \foreach \xx/\lab in {.1/{0.1},.2/{0.2},.4/{0.4},.6/{0.6},.8/{0.8},1/{1.0}} {
  \draw (\xx,0) -- (\xx,-.012) node[below] {$\lab$};
 }
 \foreach \yy/\lab in {.2/{0.2},.4/{0.4},.6/{0.6},.8/{0.8}} {
  \draw (0,\yy) -- (-.008,\yy) node[left] {$\lab$};
  \draw[black!10] (.01,\yy) -- (1,\yy);
 }
 \draw[densely dotted,black!55] (.02,.1) -- (1,.1);
 \draw[thick,teal!70!black,dashed] (.02,.08) -- (1,.08);
 \draw[very thick,blue!70!black] (.02,.8) -- (.1,.8)
   plot[domain=.1:1,samples=120] (\x,{.08/\x});
 \draw[densely dashed,blue!40] (.1,0) -- (.1,.8);
 \draw[densely dashed,black!35] (.8,0) -- (.8,.1);
 \fill[blue!70!black] (.2,.4) circle[radius=1.4pt];
 \node[anchor=west,align=left] at (.245,.43)
  {$p=0.20$, clipped bound $0.40$\\Null-reference bound $0.08$};
 \node[anchor=west,font=\footnotesize] at (.48,.78)
  {Fixed $q=0.10$, test null proportion $r=0.80$};
 \draw[very thick,blue!70!black] (.06,-.30) -- (.12,-.30);
 \node[anchor=west,font=\footnotesize] at (.135,-.30)
  {Clipped cfBH $r\min\{1,q/p\}$};
 \draw[thick,teal!70!black,dashed] (.06,-.40) -- (.12,-.40);
 \node[anchor=west,font=\footnotesize] at (.135,-.40)
  {Null-reference BH $qr$};
 \draw[densely dotted,black!55] (.60,-.40) -- (.66,-.40);
 \node[anchor=west,font=\footnotesize] at (.675,-.40)
  {Nominal FDR level $q$};
\end{tikzpicture}
\caption{Theoretical FDR bounds as the source null proportion $p$
varies, with nominal level $q=0.10$ and test null proportion $r=0.80$
held fixed. The blue curve shows the sharp limiting bound
$r\min\{1,q/p\}$ for conservative clipped cfBH. The green dashed
line shows the finite-sample bound $qr$ for null-reference BH,
and the dotted line marks the nominal level $q$.
The shaded region $p<r$ corresponds to a source population with
a higher proportion of targets than the test population.
When $p<q$, clipped cfBH selects the entire test pool with probability
tending to one as the source calibration sample grows.}
\label{fig:prevalence-envelope}
\end{figure}

To demonstrate sharpness, consider the score
$g(x)=\mathbf{1}\{x\in A\}$, which assigns a high score to
candidates in $A$ and a low score to all others.
Let $a$ be the probability of $A$ among test nulls and $t$ its
probability among true targets. Then
$d=ra+(1-r)t>0$ is the probability that a test candidate lies in $A$.

Suppose $p>q$. As $n,m\to\infty$, clipped p-values outside $A$
converge to $p>q$, so those candidates are excluded with probability
tending to one. The clipped p-values inside $A$ converge to $pa$,
while the BH threshold for selecting all candidates in $A$
converges to $qd$. Thus, if $pa<qd$, BH selects exactly the
test candidates in $A$ with probability tending to one.
The FDR then converges to
\[
 \frac{ra}{d}
 =\frac{qr}{p}\,\frac{pa}{qd}.
\]
Choosing distributions for which $pa/(qd)$ is arbitrarily close
to one from below makes this limit arbitrarily close to $qr/p$.
Each choice of distributions and score is held fixed as the sample
sizes grow.

When $p<q$, selection of the entire test pool gives limiting FDR
$r$. The boundary case $p=q$ requires a separate argument.
Appendix~\ref{app:prevalence-details} provides that argument
and the exact finite-sample calculation for the two-valued score.

\begin{corollary}[Sharp limiting bound]
\label{cor:shift-sharp}
Under Assumption~\ref{ass:shift}, fix $p\in(0,1]$, $r\in(0,1)$,
and $q\in(0,1)$. The limiting upper bound
$r\min\{1,q/p\}$ is sharp. For every $\eta>0$, there exist
fixed null and target distributions and a two-valued score for which
the clipped-cfBH FDR converges, as $n,m\to\infty$, to a value
greater than
\[
 r\min\{1,q/p\}-\eta.
\]
\end{corollary}

\subsection{Power limits from finite calibration}
\label{sec:resolution}

Null-reference BH controls FDR under changes in prevalence when source
and test null-score distributions agree. Its power, however, can be
limited by the number of calibration nulls. With $k$ calibration nulls,
the smallest possible p-value is $1/(k+1)$. Even a candidate whose score
exceeds every calibration-null score receives this p-value. To select
$R>0$ candidates, BH requires $1/(k+1)\le qR/m$. Thus, every nonempty
selection must satisfy
\[
 R\ge L=\left\lceil\frac{m}{q(k+1)}\right\rceil.
\]
For example, with $m=2000$, $q=0.10$, and $k=199$, the method must
select either no candidates or at least $100$. If the pool contains
only a few dozen targets, any nonempty selection must therefore
contain many nulls. FDR control then limits how often the method can
make a nonempty selection.

This argument is closely related to \citet[Lemma~D.1]{MaryD2022ejs}.
In our notation, their result bounds the probability of selecting any target by $2q$ when
$2q(k+1)m_1\le m$, where $m_1$ is the number of test targets.
Their proof combines the minimum possible p-value with FDR control.
The following proposition refines this argument by retaining the
integer threshold $L$ and the conditional null fraction $m_0/m$.
It holds for every fixed test label vector and allows conservative
score ties.

\begin{proposition}[A conditional resolution bound]
\label{prop:resolution}
Under Assumption~\ref{ass:shift}, fix the independent score-training
data. For every fixed calibration-null count $K=k$ and test membership
vector $\mathbf Z=(Z_1,\ldots,Z_m)^\top$, where $Z_j=1$ for a target
and $Z_j=0$ for a null, let $m_0$ and $m_1$ be the respective null
and target counts among the $m=m_0+m_1$ test candidates. For null-reference BH
at level $q$, let $R$ be the number of selected candidates and $T$
the number of selected targets. Define
\[
 L=\left\lceil\frac{m}{q(k+1)}\right\rceil,
 \qquad
 \Power_{k,\mathbf Z}
 =\E\left[\frac{T}{1\vee m_1}\,\middle|\,K=k,\mathbf Z\right].
\]
The expectation averages over the calibration and test scores,
with the score-training data held fixed.

Every nonempty selection satisfies $R\ge L$. If $m_1=0$ or $L>m$,
conditional power is zero. If $0<m_1<L\le m$, then
\begin{equation}
\label{eq:resolution-finite}
 \Power_{k,\mathbf Z}
 \le \Prob(R>0\mid K=k,\mathbf Z)
 \le \min\left\{1,
       \frac{q\,m_0/m}{1-m_1/L}\right\}.
\end{equation}
If $L\le m_1$, this argument gives only the trivial power bound of one.
\end{proposition}

To see why the bound holds, suppose $0<m_1<L$. Any nonempty selection
contains at least $L$ candidates but at most $m_1$ targets, so its
false discovery proportion is at least $1-m_1/L$. Conditional FDR
control gives
\[
 \left(1-\frac{m_1}{L}\right)
 \Prob(R>0\mid K=k,\mathbf Z)
 \le \E[\FDP\mid K=k,\mathbf Z]
 \le q\frac{m_0}{m}.
\]
Moreover, the fraction of targets selected is zero when $R=0$ and
at most one otherwise. Its expectation is therefore at most the
probability of a nonempty selection.
Appendix~\ref{app:resolution-details} gives the proof.

Now keep $k$ fixed and let the test pool grow. Suppose the test
candidates are sampled independently with null probability
$r\in(0,1)$, and let $\Power_m(k)$ denote power averaged over their
membership labels. If $q(k+1)(1-r)<1$, the law of large numbers and
\eqref{eq:resolution-finite} give
\begin{equation}
\label{eq:resolution-limit}
 \limsup_{m\to\infty}\Power_m(k)
 \le \min\left\{1,
       \frac{qr}{1-q(k+1)(1-r)}\right\}.
\end{equation}
Consequently, power can approach one only if
\begin{equation}
\label{eq:resolution-unit}
 k+1\ge\frac{1-qr}{q(1-r)}.
\end{equation}
At $q=0.10$ and target prevalence $1-r=0.01$, this requires
$k\ge900$. With fewer calibration nulls, power cannot approach one,
even with a perfectly separating score. Meeting this condition does
not guarantee power approaching one, because power still depends on
the score distributions and ties.

Falling below the threshold does not imply that power tends to zero.
Appendix~\ref{app:resolution-example} gives an example with power
bounded away from zero. If $K$ is random, apply the conditional bound
for each value of $K$ and then average over its distribution.

\section{Calibration with Random Pilot Labels}
\label{sec:pilot}

Before screening $m$ test candidates, suppose we obtain labels for an
independent pilot sample of $b$ observations from the test population.
Let $B$ be the number of nulls in this sample. Since each pilot observation
is independently drawn from a population with null proportion $r$,
$B\sim\mathrm{Bin}(b,r)$.
The pilot requires $b$ labels, of which $B$ identify null observations
that can be used for calibration.
The pilot observations are separate from the $m$ test candidates,
and the scoring rule remains fixed.

The pilot serves two purposes. Its $B$ null scores provide a calibration
reference, while all $b$ membership labels provide information about the
test-population null proportion $r$. Both procedures below use this
label information through the same BH level $\alpha(B)$ defined in
\eqref{eq:pilot-level}. As the pilot size increases, this level approaches
$\min\{1,q/r\}$. Pooling at this level requires matching source
and test null-score distributions. Pilot-only calibration avoids
that requirement.

\subsection{Pooling representative source nulls}
\label{sec:pilot-pooled}

Under Assumption~\ref{ass:shift}, source, pilot, and test null scores
have the same distribution. We can therefore use the $K$ source nulls
and $B$ pilot nulls together as the calibration reference.

Let $S_1^{\rm s},\ldots,S_K^{\rm s}$ denote the source null scores and
$S_1^{\rm p},\ldots,S_B^{\rm p}$ denote the pilot null scores.
For test candidate $j$, define
\begin{equation}\label{eq:pilot-ranks}
 p_j^+
 =
 \frac{
  1
  +\sum_{i=1}^{K}\ind\{S_i^{\rm s}\ge g(X_j^{\rm t})\}
  +\sum_{\ell=1}^{B}\ind\{S_\ell^{\rm p}\ge g(X_j^{\rm t})\}
 }{K+B+1}.
\end{equation}
The two sums count source nulls and pilot nulls whose scores are
at least as large as the test candidate's score. The added one
accounts for the test candidate itself. A candidate scoring above
every calibration null therefore receives the smallest possible
p-value, $1/(K+B+1)$.

We apply BH to these p-values at level $\alpha(B)$ from
\eqref{eq:pilot-level}. Compared with using the pilot nulls alone,
adding source nulls changes the smallest possible p-value from
$1/(B+1)$ to $1/(K+B+1)$. Both procedures use the same $b$ pilot
labels and the same testing level. The source nulls enlarge the
calibration reference, but the level adjustment still uses only
the pilot labels to account for the test-population null proportion.
The finer p-value resolution can improve power, although a power gain
is not guaranteed.

\begin{proposition}[FDR control with pooled source and pilot nulls]
\label{prop:pilot-fdr}
Under the sampling conditions above and with an independently fitted
score, BH applied to \eqref{eq:pilot-ranks} at level $\alpha(B)$ satisfies
\begin{equation}
\label{eq:pilot-fdr}
 \FDR\le r\E\left[\min\left\{1,\frac{q(b+1)}{B+1}\right\}\right]
 \le q\{1-(1-r)^{b+1}\}\le q.
\end{equation}
The guarantee permits conservative score ties and applies to selections
among the $m$ test candidates, excluding the pilot observations.
\end{proposition}

Conditional on all membership labels, the reference size $K+B$
and testing level $\alpha(B)$ are fixed.
Lemma~\ref{lem:conditional-bh} then bounds conditional FDR by
$\alpha(B)m_0/m$, where $m_0$ is the number of null test candidates.
The expectation still averages over the calibration and test scores.
Since the pilot and test labels are independent and
$\E[m_0/m]=r$, averaging over the labels gives the first inequality
in \eqref{eq:pilot-fdr}. The remaining inequalities follow from
$\alpha(B)\le q(b+1)/(B+1)$ and the binomial identity
\[
 r(b+1)\E\left[\frac1{B+1}\right]
 =1-(1-r)^{b+1}.
\]
Thus, the guarantee averages over the random pilot. It does not
assert FDR control at level $q$ for every realized pilot.

Equivalently, the procedure applies BH at level $q$ to the scaled
values $[(B+1)/(b+1)]p_j^+$. These values express the level adjustment
and need not themselves be valid null p-values.

\subsection{Using pilot nulls when source nulls may be biased}
\label{sec:pilot-only}

When source and test null scores have different distributions, source
nulls may be unsuitable for calibration. We can instead use only the
nulls from a representative pilot sample.

Let $S_1^{\rm p},\ldots,S_B^{\rm p}$ denote the pilot null scores.
For test candidate $j$, define
\begin{equation}\label{eq:pilot-only-ranks}
 u_j=
 \frac{1+\sum_{\ell=1}^{B}
 \ind\{S_\ell^{\rm p}\ge g(X_j^{\rm t})\}}{B+1},
 \qquad
 c_j=\frac{B+1}{b+1}u_j.
\end{equation}
The sum counts pilot nulls whose scores are at least as large as the
test candidate's score. Apply BH to $u_j$ at level $\alpha(B)$ from
\eqref{eq:pilot-level}. Equivalently, apply BH to the scaled values
$c_j$ at level $q$. This is the clipped-cfBH construction applied
to the labeled pilot.

The guarantee assumes that the score is fitted independently of the
pilot and test data and that pilot and test observations are drawn
independently from the same test population. Source data may be used
to train the score, but source null scores are excluded from the
calibration reference. No agreement between source and test null-score
distributions is required.

\begin{corollary}[FDR control with pilot-only calibration]
\label{cor:pilot-only}
Under the sampling conditions above, for a fixed integer $b\ge0$,
BH applied to $u_j$ at level $\alpha(B)$, or equivalently to $c_j$
at level $q$, satisfies
\begin{equation}
\label{eq:pilot-only-fdr}
 \FDR\le r\E\!\left[\min\left\{1,\frac{q(b+1)}{B+1}\right\}\right]
 \le q\{1-(1-r)^{b+1}\}\le q.
\end{equation}
The guarantee permits conservative score ties and applies to the
$m$ test candidates, excluding the pilot observations.
\end{corollary}

The proof follows the argument for
Proposition~\ref{prop:pilot-fdr}, using $B$ calibration nulls in place
of $K+B$. Conditional on the membership labels,
Lemma~\ref{lem:conditional-bh} bounds FDR by $\alpha(B)m_0/m$.
Averaging over the pilot and test labels gives
\eqref{eq:pilot-only-fdr}.

Using only pilot nulls avoids relying on source compatibility, but
provides fewer calibration scores. The smallest possible value of
$u_j$ is $1/(B+1)$, and the smallest possible scaled value $c_j$ is
$1/(b+1)$. If $q(b+1)<1$, every $c_j$ exceeds $q$. Since no BH
threshold exceeds $q$, no selection is possible.

Pooling and pilot-only calibration use the same $b$ pilot labels.
Pooling at the common level $\alpha(B)$ relies on matching source
and test null-score distributions. Their individual guarantees do
not justify choosing whichever
procedure produces more selections after examining the results.
The analysis also assumes a fixed pilot size, random pilot sampling,
and an independently fitted score. Selecting pilot observations
based on their features, stopping based on observed pilot labels,
or using the pilot to fit the score requires a separate justification.

\section{Null-Score Distribution Mismatch}
\label{sec:null-robustness}

The guarantees for source calibration require source and test null
scores to have the same distribution. Pilot-only calibration avoids
this requirement by using a representative sample from the test
population. We now examine what happens when source nulls are still
used despite a mismatch.

Such a mismatch can arise when source observations are collected
according to molecular structure, institution, measurement conditions,
or a previous model's predictions. Even with the same scoring rule,
source nulls may then receive systematically different scores from
test nulls. We first bound the resulting FDR inflation and then
examine how a representative pilot can protect against it. Throughout
this section, the score is fitted independently of calibration and
test data and remains fixed.

\subsection{A quantitative sensitivity bound for source calibration}
\label{sec:score-null-shift}

Let $P_0^g$ and $Q_0^g$ denote the distributions of the score $g(X)$
for source and test nulls. We measure their difference using
\[
 \varepsilon_g=\operatorname{TV}(Q_0^g,P_0^g)
 =\sup_A|Q_0^g(A)-P_0^g(A)|,
 \qquad
 \Gamma_g=\left\|\frac{dQ_0^g}{dP_0^g}\right\|_\infty.
\]
Here, $\varepsilon_g$ is the largest difference between the
probabilities that a source null score and a test null score fall
in the same measurable set $A$. When finite, $\Gamma_g$ bounds the
relative increase in these probabilities, so that
$Q_0^g(A)\le\Gamma_g P_0^g(A)$ for every $A$.
The density ratio is defined when $Q_0^g\ll P_0^g$, and its norm is
the essential supremum under $P_0^g$. Otherwise, we set
$\Gamma_g=\infty$.

Both quantities compare score distributions. They do not require
the full covariate distributions to agree.
A related use of total variation appears in conformal prediction.
\citet{BarberRF2023aos} bound coverage loss using total variation
between the data distribution and distributions obtained by swapping
observations. Here, we bound FDR, and the minimum number of candidates
that BH can select determines the sensitivity to mismatch.
The density-ratio bound follows from a standard change-of-measure
argument.

\begin{theorem}[Finite-sample FDR sensitivity to null-score shift]
\label{thm:score-null-shift}
Fix an independently fitted score and a test membership label vector
$\mathbf Z$ containing $m_0$ nulls among $m$ candidates.
Conditional on these labels, suppose the $k$ source-null scores are
i.i.d. from $P_0^g$, the test null scores are i.i.d. from $Q_0^g$,
and the two samples are independent. The test target scores are
independent of both samples.

Apply conservative null-reference BH at a fixed level
$0<\alpha<1$, and define
\[
 L(k,\alpha)=\left\lceil\frac{m}{\alpha(k+1)}\right\rceil.
\]
If $L(k,\alpha)>m$, no selection is possible. Otherwise,
\begin{align}
 \E[\FDP\mid\mathbf Z]
 &\le \frac{m_0}{m}\,\Gamma_g\alpha,
 \label{eq:lr-robust}\\
 \E[\FDP\mid\mathbf Z]
 &\le \frac{m_0}{m}
   \left\{\alpha+\varepsilon_g\frac{m}{L(k,\alpha)}\right\}
 \le \frac{m_0}{m}\alpha\{1+(k+1)\varepsilon_g\}.
 \label{eq:tv-robust}
\end{align}
The first bound applies when $\Gamma_g<\infty$. The second does
not require absolute continuity. The expectations average over
the calibration and test scores.

At $\alpha=1$, all candidates are selected and the conditional FDR
equals $m_0/m$. The bounds also hold when ties are broken
lexicographically using independent uniform marks.
\end{theorem}

The total-variation bound has a simple interpretation. Every
nonempty selection contains at least $L(k,\alpha)$ candidates,
so each selected null contributes at most $1/L(k,\alpha)$ to the
FDP. For each test null, compare its actual score distribution
with $P_0^g$, leaving all other score distributions unchanged.
Under this comparison, the conformal BH argument bounds its
expected FDP contribution by $\alpha/m$. Changing its distribution
back to $Q_0^g$ adds at most $\varepsilon_g/L(k,\alpha)$.
Summing over the $m_0$ test nulls gives \eqref{eq:tv-robust}.
The density-ratio argument instead multiplies each contribution
bound by at most $\Gamma_g$. Appendix~\ref{app:null-robustness}
gives the proof.

At $\alpha=q$, the minimum selection size is the same $L$ as in
Proposition~\ref{prop:resolution}. Thus, the quantity that limits
power also controls sensitivity to mismatch. More calibration
nulls allow smaller selections, in which a single false discovery
can account for a larger fraction of the selected candidates.

For example, when $k=999$ and $\varepsilon_g\le10^{-4}$, the
simplified bound in \eqref{eq:tv-robust} is at most $1.10$ times
the bound $\alpha m_0/m$ under matching null-score distributions.
The sharper additive term uses $m/L(k,\alpha)$ and accounts for
the finite test-pool size. Proposition~\ref{prop:shift-breakdown}
below shows that a mismatch of order $1/(k+1)$ can nevertheless
drive FDR to one in sufficiently large test pools.

Clipped cfBH is null-reference BH at the random level
$\alpha_K=\min\{1,q(n+1)/(K+1)\}$. Applying the theorem conditional
on $K$ and then averaging gives the following result.

\begin{corollary}[Prevalence and null-score mismatch together]
\label{cor:joint-shift}
Retain the independent sampling assumptions and source and test
null proportions $p,r$, but allow $Q_0^g\ne P_0^g$. Let
$K\sim\mathrm{Bin}(n,p)$, and define
\[
 \alpha_K=\min\{1,q(n+1)/(K+1)\},
 \qquad
 L_K=L(K,\alpha_K).
\]
Source-only clipped cfBH satisfies
\begin{align}
 \FDR_{\mathrm{clip}}
 &\le \min\{1,\Gamma_g r\E[\alpha_K]\},
 \label{eq:joint-lr}\\
 \FDR_{\mathrm{clip}}
 &\le \min\!\left\{1,r\E\left[
   \alpha_K+\varepsilon_g\frac{m}{L_K}\ind\{\alpha_K<1\}
   \right]\right\}.
 \label{eq:joint-tv}
\end{align}
The first bound applies when $\Gamma_g<\infty$.
For null-reference BH at level $q$, the corresponding bounds
follow by replacing $\alpha_K$ with $q$ throughout.
When the null-score distributions agree, $\Gamma_g=1$ and
$\varepsilon_g=0$, and both bounds reduce to the prevalence bound
$r\E[\alpha_K]$ in Proposition~\ref{prop:shift-bh}.
\end{corollary}

The indicator in \eqref{eq:joint-tv} accounts for the case
$\alpha_K=1$. In that case, every candidate is selected, so FDR
depends only on the null proportion and no mismatch penalty
is needed.

A small fraction of unusually high-scoring test nulls can receive
the smallest possible p-values. If enough enter the test pool,
they can trigger a selection consisting entirely of nulls.
The following construction shows that the dependence on calibration
size in the sensitivity bound reflects a possible failure of the
method. Here, contamination changes the test null distribution.
\citet{BashariM2025icml} instead consider contamination of the
calibration reference.

\Needspace{5\baselineskip}
\begin{proposition}[Rare null contamination can destroy FDR control]
\label{prop:shift-breakdown}
Fix $q,r\in(0,1)$ and $k$ with $q(k+1)>1$.
Let the $k$ source-null scores be independent draws from $U(0,1)$.
Independently generate $m$ test candidates, each null with
probability $r$. Every test target has score $-1$, and each test
null has score distribution
\[
 Q_0^g=(1-\varepsilon)U(0,1)+\varepsilon\,\delta_2,
 \qquad
 \operatorname{TV}(Q_0^g,P_0^g)=\varepsilon,
 \qquad 0\le\varepsilon\le1,
\]
where $\delta_2$ is a point mass at score $2$.

For null-reference BH at level $q$,
\begin{equation}
 \label{eq:contamination-lower}
 \FDR\ge
 \Prob\!\left\{\operatorname{Bin}(m,r\varepsilon)
 \ge\left\lceil\frac{m}{q(k+1)}\right\rceil\right\}.
\end{equation}
If $r\varepsilon>1/[q(k+1)]$, then $\FDR\to1$ as $m\to\infty$.
In particular, fix any $c>1/(qr)$ and set
$\varepsilon_k=c/(k+1)$ for sufficiently large $k$. Then
$\varepsilon_k\to0$, but
\[
 \lim_{k\to\infty}\lim_{m\to\infty}\FDR=1.
\]
\end{proposition}

Every contaminated null has score $2$ and receives p-value
$1/(k+1)$. If at least $L$ such nulls appear, they satisfy the
BH selection threshold. Every target has p-value one and is
excluded, so any selection contains only nulls.

For example, take $k=999$, $m=2000$, $r=0.80$, and $q=0.10$.
Then $L=20$. If 1.5\% of test nulls belong to the high-score
component, the number of contaminated candidates has distribution
$\operatorname{Bin}(2000,0.012)$. Its probability of being at least
20 is approximately $0.8213$, so FDR is at least 82.1\%.
This construction demonstrates worst-case sensitivity to a small
fraction of unusually high-scoring nulls.

\subsection{The cost of borrowing unrepresentative source nulls}
\label{sec:unmatched-pool}

A representative pilot supports FDR control without using source nulls.
We now ask what happens if we also include $K$ source null scores
whose distribution may differ from that of test nulls. The result
below bounds the cost of including these scores without assuming
that they are representative.

We require the $B$ pilot null scores and each individual test null
score to be exchangeable conditional on the source scores, all
membership labels, and the other test scores. In other words,
swapping that test null score with any pilot null score must leave
their conditional distribution unchanged. This condition holds
for independent pilot and test samples from the same test population.
It also holds for the uniform sampling without replacement used
in Section~\ref{sec:exp-aqsol}, as shown in
Appendix~\ref{app:finite-library}.

The source reference is fixed before the pilot and test observations
are allocated. Its scores may otherwise be arbitrary, subject to
the conditional exchangeability requirement above. We use the
pooled p-values $p_j^+$ from \eqref{eq:pilot-ranks}.

\begin{proposition}[FDR bound for arbitrary source-null pooling]
\label{prop:unmatched-pooling}
Fix $K,B$ and a test membership label vector containing $m_0$ nulls
among $m$ candidates. Under the conditional exchangeability
assumption above, BH applied to the pooled p-values at a fixed
level $\alpha\in(0,1)$ satisfies
\begin{equation}
 \label{eq:unmatched-conditional}
 \E[\FDP\mid\text{source scores, membership labels}]
 \le \frac{m_0}{m}\,\alpha\frac{K+B+1}{B+1}.
\end{equation}
The expectation averages over the pilot and test scores.
At $\alpha=1$, all candidates are selected and the conditional
FDR equals $m_0/m$.

Suppose each test membership label has null probability $r$
independently of $(K,B)$, and use
$\alpha(B)=\min\{1,q(b+1)/(B+1)\}$. Then
\begin{equation}
 \label{eq:unmatched-unconditional}
 \FDR_{\mathrm{pooled}}
 \le r\E\!\left[
 \ind\{\alpha(B)=1\}
 +\ind\{\alpha(B)<1\}
 \alpha(B)\frac{K+B+1}{B+1}\right].
\end{equation}

For a single null test candidate, the conditional bound can be
attained. Specifically, suppose the pilot and test null scores
have no ties and all source scores lie below them. The rejection
probability is
\[
 \min\left\{1,
 \frac{\lfloor\alpha(K+B+1)\rfloor}{B+1}\right\}.
\]
This equals the bound in \eqref{eq:unmatched-conditional} whenever
$\alpha(K+B+1)$ is an integer not exceeding $B+1$.
\end{proposition}

Figure~\ref{fig:null-mismatch-bounds}(a) shows how this upper bound
changes with pilot size. Panel (b) shows the contamination lower
bound from Proposition~\ref{prop:shift-breakdown}.

\begin{figure}[!htbp]
 \centering
 \includegraphics[width=\linewidth]{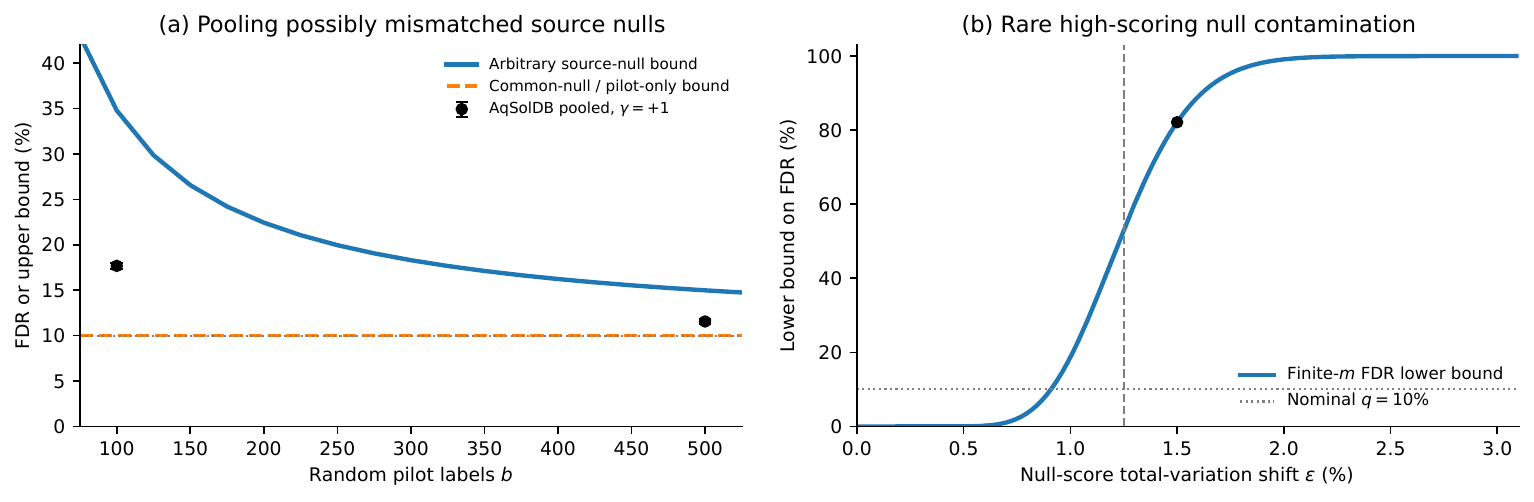}
 \caption{Finite-sample FDR bounds under null-score mismatch.
 (a) The upper bound for arbitrary source-null pooling with
 $(n,p,r,q)=(999,0.2,0.8,0.1)$, together with the bound for
 matching null-score distributions or pilot-only calibration.
 The two points show empirical FDR estimates from the
 log-P-biased AqSolDB experiment, with Monte Carlo standard-error
 bars. (b) The lower bound for the rare-contamination construction
 with $k=999$, $m=2000$, $r=0.8$, and $q=0.1$.
 The dotted line marks the nominal FDR level of 10\%.
 The curves are theoretical bounds, not actual FDR values.}
 \label{fig:null-mismatch-bounds}
\end{figure}

The factor $(K+B+1)/(B+1)$ multiplies the usual conditional FDR
bound. It equals one when $K=0$ and grows when many source nulls
are added to a small pilot reference. Thus, adding calibration
scores without a compatibility assumption can substantially
weaken the guarantee. The factor describes an upper bound on
FDR and does not assert that actual FDR increases by this amount.

Under independent source and pilot sampling,
$K\sim\mathrm{Bin}(n,p)$ and $B\sim\mathrm{Bin}(b,r)$ are independent.
We can therefore evaluate \eqref{eq:unmatched-unconditional}
from these two count distributions without estimating the
difference between source and test null-score distributions.

For the AqSolDB setting with $n=999$, $p=0.20$, $r=0.80$, and
$q=0.10$, the resulting FDR upper bounds are 34.79\% for $b=100$
and 14.99\% for $b=500$. The corresponding empirical FDR estimates
under log-P-biased source sampling with $\gamma=+1$ are 17.68\%
and 11.56\%, respectively
(Table~\ref{tab:aqsol-main}). Increasing the pilot size reduces
the upper bound, but neither bound guarantees FDR control at
the nominal 10\% level in these settings.

We can restore FDR control by dividing the testing level by
$(K+B+1)/(B+1)$. However, this correction also removes any
selection advantage over using the pilot alone.

\begin{corollary}[FDR control and containment after level correction]
\label{cor:pool-containment}
Under the assumptions of Proposition~\ref{prop:unmatched-pooling},
let $u_j$ be the pilot-only null-reference p-value and $p_j^+$
the p-value obtained by combining the $B$ pilot nulls with
$K$ arbitrary source nulls. For any pilot testing level
$\alpha(B)\in[0,1]$, define
\[
 \alpha_{\rm corr}
 =\alpha(B)\frac{B+1}{K+B+1}.
\]
Use conservative ties in both procedures, or the same independent
tie-breaking marks. Then
\[
 p_j^+\ge\frac{B+1}{K+B+1}u_j,
 \qquad
 \mathcal S_{\rm pooled}(\alpha_{\rm corr})
 \subseteq\mathcal S_{\rm pilot}(\alpha(B)),
\]
where $\mathcal S$ denotes the selected test candidates.

Under the random-pilot sampling conditions above, taking
$\alpha(B)=\min\{1,q(b+1)/(B+1)\}$ gives
\[
 \FDR_{\rm corr}\le r\E[\alpha(B)]\le q.
\]
Thus, corrected pooling controls FDR without matching source and
test null-score distributions, but every candidate it selects
is also selected by pilot-only calibration.
\end{corollary}

The two conclusions follow from separate arguments. Substituting
$\alpha_{\rm corr}$ into
Proposition~\ref{prop:unmatched-pooling} cancels the inflation
factor and gives FDR control. For containment, adding source
scores contributes a nonnegative count to each p-value numerator,
which gives the displayed p-value inequality. BH therefore
selects a subset of the pilot-only selections.
Appendix~\ref{app:pooling-proof} gives the details.

For this pooled-rank procedure, the correction cannot improve
power over pilot-only calibration because it can only remove
selected candidates. This conclusion concerns the stated
correction. It does not rule out benefits from source data
under additional compatibility assumptions or through other
selection procedures.

The bounds involving $\Gamma_g$ or $\varepsilon_g$ require
justified upper bounds on these distributional discrepancies.
A small pilot sample does not by itself provide such bounds.
The arbitrary-source pooling bound avoids this requirement,
but may be conservative. These guarantees apply to the
specified procedures and do not justify choosing whichever
procedure produces more selections after examining the results.

\section{Experiments}
\label{sec:experiments}

We evaluate the effects of changing class proportions, limited
calibration nulls, and unrepresentative source nulls. The nominal
FDR level is $q=0.10$ throughout. We estimate FDR and power by
averaging the FDP and the fraction of test targets selected across
replications. Uncertainties in parentheses are Monte Carlo
standard errors (MCSEs).  Comparisons within a setting use paired draws.
Synthetic source, pilot, and test samples are independent. For frozen-score
methods, MCSEs describe resampling variability conditional on the fitted
score and, for AqSolDB, the evaluation library. ECOT-bi is refitted in
each resample.
Section~\ref{sec:exp-implementation} specifies the rank conventions
and comparator implementations.

\subsection{A controlled check with changing class proportions}
\label{sec:exp-shift}

First, we hold the score distribution within each class fixed and
change only $p$ and $r$.  Null scores originate from $N(0,1)$ and target
scores from $N(3,1)$. We use eight bins with boundaries
$(-\infty,-1,0,1,1.5,2,2.5,3,+\infty)$. Each bin's score is the ratio
of its target probability to its null probability. Clipped cfBH and
null-reference BH use this same score with conservative ties. Each cell
uses $n=999$ source observations, $m=2000$ test candidates, and 20,000
paired replications with master seed $20261110$. Conditional bin counts
are sampled from their exact multinomial distributions.

\begin{table}[!htbp]
\centering\small
\caption{A fixed eight-bin score when only class proportions change. FDR and power are percentages, with MCSEs in parentheses. All cells use 20,000 paired replications. Clip denotes clipped cfBH, and Null-BH denotes conservative null-reference BH.}
\label{tab:shift-bh}
\begin{tabular}{rr rrrr}
\toprule
$p$ & $r$ & Clip FDR & Clip power & Null-BH FDR & Null-BH power \\
\midrule
0.02 & 0.80 & 80.00 (0.01) & 100.00 (0.00) & 0.00 (0.00) & 0.01 (0.01) \\
0.20 & 0.80 & 30.40 (0.06) & 95.37 (0.02) & 3.97 (0.02) & 62.75 (0.15) \\
0.20 & 0.95 & 31.88 (0.12) & 79.44 (0.10) & 0.52 (0.03) & 1.67 (0.08) \\
0.80 & 0.50 & 3.77 (0.01) & 86.70 (0.03) & 2.71 (0.01) & 84.29 (0.01) \\
0.80 & 0.80 & 6.39 (0.02) & 76.07 (0.06) & 4.44 (0.02) & 71.22 (0.05) \\
\bottomrule
\end{tabular}
\end{table}

Table~\ref{tab:shift-bh} shows the clearest contrast at $p=0.20$
and $r=0.80$. The source has a much higher
target proportion than the screening pool.  Clipped cfBH makes many more discoveries,
but its FDR rises to 30.40\%, far above the nominal 10\% level.  Null-reference
BH keeps FDR at 3.97\%, with lower power (62.75\% versus 95.37\%).
The higher power of clipped cfBH comes with a violation of the
FDR constraint.  When $p=0.02<q$, the clipped method
selects the whole pool in these simulations, matching the saturation
result from Section~\ref{sec:shift}.  The $p\ge r$ cells satisfy the
condition of Corollary~\ref{cor:safe}.

The very low null-reference power at $p=0.20$, $r=0.95$ illustrates
the resolution effect discussed in Section~\ref{sec:resolution}.
Limited null references, overlapping scores, and discretization jointly
restrict how many targets can pass BH.

\subsection{Molecular screening with and without source calibration}
\label{sec:exp-aqsol}

We use AqSolDB \citep{SorkunM2019sd} to identify molecules with measured
solubility $\mathrm{LogS}\ge-1$. Removing two unparsable molecules leaves
9,980 distinct canonical isomeric SMILES identities. A stratified 30\%
training split contains 2,994 identities. We fit a random forest with
300 trees, minimum leaf size two, square-root feature subsampling, and
balanced class weights to 1,024-bit radius-two RDKit Morgan
fingerprints.\footnote{RDKit, Open-Source Cheminformatics Software,
\url{https://www.rdkit.org/}}
The frozen score has held-out AUC $0.8992$.
The evaluation library contains 5,446 nulls and 1,540 targets.

Each condition uses 1,000 paired replications with $p=0.20$, $r=0.80$,
and $n=999$ source observations. We draw $K\sim\mathrm{Bin}(999,p)$ source nulls and
$999-K$ source targets without replacement within each class,
independently between classes. Source
nulls are sampled uniformly when $\gamma=0$. For $\gamma=\pm1$, their
sampling weights are proportional to $\exp\{\gamma z_{\log P}(x)\}$,
where $z_{\log P}$ is the supplied MolLogP descriptor standardized over
all 9,980 molecules.

The 500-item pilot has independent binary membership labels with null
probability $r$. Independently, we draw the test null count from
$\mathrm{Bin}(2000,r)$. This is equivalent to generating 2,000 independent
test labels for the FDR and power calculations, because all procedures
are invariant to test ordering. Pilot and test identities are then
allocated uniformly without replacement within each remaining class,
independently between classes.
The first $b=100$ or $500$ pilot observations give nested fixed budgets.
All 500 pilot identities are excluded from testing, and identities are
disjoint across training, source calibration, pilot, and test sets.
Removing biased source identities also changes the remaining test
population, so this is a finite-library stress test.
Proposition~\ref{prop:finite-library} in Appendix~\ref{app:finite-library}
justifies the pilot-only guarantee, and the pooled guarantee when
$\gamma=0$, under this sampling design. No requested allocation was infeasible
in the reported replications. The proposition specifies an empty
selection on that event to define the procedure for all possible draws.

The ECOT-bi comparator uses the fingerprints and frozen score as features.
Its logistic classifier uses $C=1$, the liblinear solver, tolerance
$10^{-6}$, and at most 1,000 iterations. The common fitting and ranking
rules are given in Section~\ref{sec:exp-implementation}.

\begin{table}[!htbp]
\centering\small
\caption{Molecular screening with distinct identities at $p=0.20$ and $r=0.80$.
Each entry is FDR/power in percent from the primary 1,000-replication
results at $q=0.10$. Monte Carlo standard errors are at most $0.46$ for FDR and $0.77$ for
power, measured in percentage points. $\gamma=0$ is uniform source
sampling. $\gamma=\pm1$ biases source-null sampling by standardized log-P.  Clipped and adaptive rows use
coherent tie marks. ECOT-bi is an additional pooled-score comparator.}
\label{tab:aqsol-main}
\begin{tabular}{lrrr}
\toprule
Method & $\gamma=0$ & $\gamma=+1$ & $\gamma=-1$ \\
\midrule
Source null-reference BH & 4.02 / 8.04 & 20.61 / 40.46 & 0.03 / 0.06 \\
Source clipped cfBH & 38.95 / 71.87 & 61.85 / 93.64 & 10.53 / 22.33 \\
Source Storey & 6.63 / 13.39 & 35.62 / 66.12 & 0.02 / 0.04 \\
Pooled add-one, $b=100$ & 7.84 / 15.64 & 17.68 / 35.21 & 1.88 / 3.84 \\
Pooled add-one, $b=500$ & 8.46 / 16.09 & 11.56 / 22.60 & 3.97 / 7.75 \\
Pilot-only clipped, $b=100$ & 3.40 / 6.92 & 3.43 / 7.05 & 3.91 / 8.60 \\
Pilot-only clipped, $b=500$ & 7.70 / 14.74 & 7.28 / 13.93 & 7.66 / 16.09 \\
ECOT-bi (pooled score fit) & 2.48 / 3.35 & 16.74 / 23.49 & 0.01 / 0.02 \\
\bottomrule
\end{tabular}
\end{table}

Table~\ref{tab:aqsol-main} shows how source compatibility and pilot size
affect FDR and power. With uniform source sampling ($\gamma=0$), a small pilot does not have
many nulls. Pooled add-one has 15.64\% power at $b=100$, compared with
6.92\% when the pilot is used alone. At $b=500$, the gap is smaller,
with powers of 16.09\% and 14.74\%, respectively. This is the cost of
discarding useful source nulls when few new labels are available.

With biased source nulls ($\gamma=+1$), source null-reference BH reaches
20.61\% FDR and pooled add-one reaches 17.68\% ($b=100$) or 11.56\%
($b=500$).  Pilot-only clipping remains below the nominal 10\% in these
replications, at 3.43\% and 7.28\%, respectively.  Source clipping can
also be anti-conservative when the source is enriched for targets even
with unbiased null sampling. Its FDR is 38.95\% at $\gamma=0$.
The primary AqSolDB output was reproduced byte for byte using
Python 3.13.18, NumPy 2.5.3, SciPy 1.18.1, pandas 3.0.5,
scikit-learn 1.9.1, RDKit 2026.09.1, and master seed $20261110$.

\subsection{A Gaussian experiment shows the cost of null mismatch}
\label{sec:exp-null-shift}

To change just the null scores, let the frozen score be $g(x)=x$.
Source nulls follow $N(0,1)$, test nulls follow $N(\delta,1)$, and
all target scores follow $N(3,1)$.  The source/test null fractions are
$(p,r)=(0.20,0.80)$, with $n=999$, $m=2000$, and a random pilot of
$b=100$ or $500$ labels.  Each cell has 2,000 paired replications.
The null distributions match when $\delta=0$ and differ otherwise.
The two pilot budgets are nested within an independent 500-observation
pilot. This experiment and the mixture experiment below use master
seed $2026111013$ and common tie marks across methods.

\begin{figure}[!htbp]
\centering
\includegraphics[width=.96\linewidth]{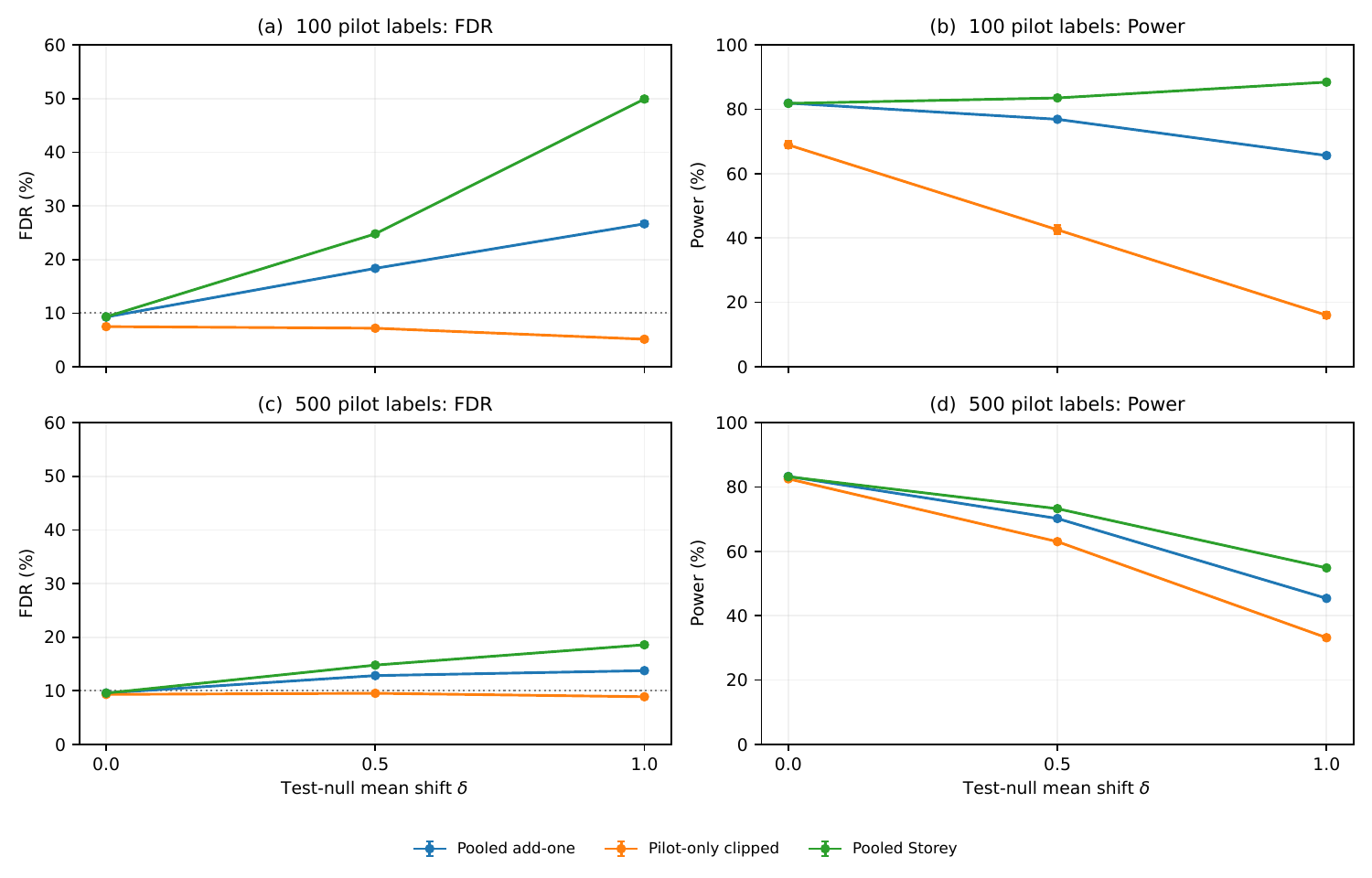}
\caption{Null-distribution shift and the pilot budget for pooled add-one,
pooled Storey, and pilot-only clipped cfBH, based on 2,000 paired
replications per cell. The columns show FDR and power. The rows compare
$b=100$ with $b=500$. Dotted lines denote the nominal FDR level $q=0.10$.
Error bars are $\pm1.96$ Monte Carlo standard errors.  Pooling can improve power when source and test null scores match,
but can inflate FDR when they differ. Pilot-only clipping uses
representative test-population data.}
\label{fig:null-shift}
\end{figure}

Figure~\ref{fig:null-shift} shows that, with $b=100$ and no null shift,
pooling raises power from 68.95\% to
81.95\%.  At $\delta=1$, however, pooled add-one has 26.65\% FDR,
compared with 5.13\% for pilot-only clipping.  With $b=500$ and
$\delta=0$, their powers are almost the same (83.31\% and 82.56\%).
The benefit of retaining the source is therefore largest when the pilot
is small and the source nulls are trustworthy.  Pooled Storey is even
more anti-conservative under severe null shift, reaching about 50\% FDR
at $b=100$, $\delta=1$.

\subsection{When a pilot-based prevalence estimate helps}
\label{sec:exp-mixture}

The next controlled example isolates the BH-level adjustment from the
additional reference scores.  Source and test null scores are uniform
on $[0,1]$.  Source targets are uniform on $[2,3]$, while test targets
are a mixture in which half look like nulls and half like source targets.
This is written $P_1^{\rm test}=aP_0+(1-a)U(2,3)$ with $a=0.50$.
All augmented methods use the \emph{same} source-plus-pilot reference
with $b=500$. Source clipped cfBH instead uses no pilot labels.
Here $p=0.20$, $n=999$, and $m=2000$, with 2,000 paired replications.

\begin{table}[!htbp]
\centering\small
\caption{FDR and power in the overlapping mixture ($a=0.5$), from
2,000 paired replications. Entries are FDR/power in percent.
Source-only clipped cfBH uses no pilot labels. Storey and pilot add-one
use the same pooled reference with $b=500$ pilot labels.
The nominal FDR level is $10\%$ and $p=0.20$.}
\label{tab:mixture}
\begin{tabular}{rrrr}
\toprule
Test null fraction $r$ & Source clipped & Storey (augmented) & Pilot add-one \\
\midrule
0.20 & 9.87 / 64.20 & 3.26 / 53.66 & 9.99 / 64.45 \\
0.30 & 14.84 / 62.95 & 4.48 / 52.89 & 9.89 / 57.40 \\
0.50 & 24.83 / 60.00 & 6.49 / 51.74 & 9.85 / 52.83 \\
0.80 & 39.27 / 54.74 & 8.35 / 50.56 & 9.56 / 50.65 \\
0.95 & 43.74 / 51.28 & 7.26 / 49.84 & 7.62 / 49.84 \\
\bottomrule
\end{tabular}
\end{table}

Table~\ref{tab:mixture} compares the three adjustments.
At $r=0.20$, source-only clipping is already valid and has almost the
same power as pilot add-one.  Thus the large power gap over Storey in
that cell is \emph{not} evidence that pilot labels are necessary.
At $r=0.30$, source clipping has 14.84\% FDR, while pilot add-one has
9.89\% FDR and 57.40\% power, compared with 52.89\% power for Storey.
When targets are rare ($r=0.80$ or $0.95$), the additional adjustment
gives little additional power in this mixture. Here the comparison
with Storey isolates the value of pilot labels for adjusting the
testing level, since both methods use the same pooled null reference.
Figure~\ref{fig:mixture} shows these trends across test null proportions
and includes the known-$r$ comparison.

\begin{figure}[!htbp]
\centering
\includegraphics[width=.98\linewidth]{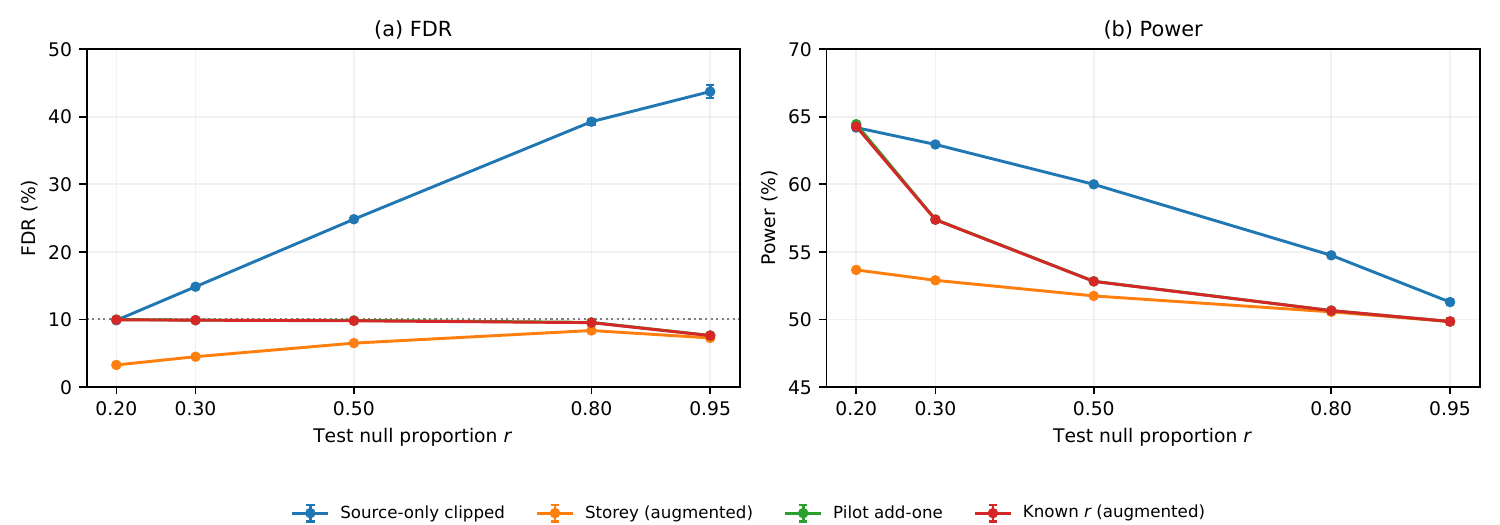}
\caption{FDR and power versus the test null fraction in the mixture
experiment ($a=0.5$). Error bars are $\pm1.96$ Monte Carlo standard errors
from 2,000 replications. Source-only clipping can match pilot adjustment
at $p=r=0.20$ without additional labels, but becomes anti-conservative
under source enrichment. Pilot add-one remains below or near the nominal FDR level, and
its gain over Storey becomes small when targets are rare.  Known-$r$ BH
uses the augmented reference and the true $r$. Its curve almost coincides
with, and hides, that of pilot add-one.}
\label{fig:mixture}
\end{figure}

\subsection{A check with an independently learned score}
\label{sec:exp-learned}

Finally, we verify that the prevalence effect persists with a learned
score. In training, calibration, and testing, features follow
$X\mid Z=0\sim N(0,I_5)$ and $X\mid Z=1\sim N(2.5e_1,I_5)$, where
$e_1=(1,0,0,0,0)^\top$. An independent training sample contains 6,000
observations with target probability $1/2$. Logistic regression uses
$C=1$, L-BFGS, tolerance $10^{-9}$, and at most 500 iterations. Its
linear decision score is frozen before calibration and testing at
$p=0.20$, $r=0.80$, $n=999$, and $m=2000$.

Table~\ref{tab:learned-gaussian} compares all methods on the same
500 paired resamples with master seed $26031109$. For the frozen score,
clipped cfBH has $39.89\%$ FDR. Null-reference BH has $6.61\%$ FDR and
$51.15\%$ power. Storey's correction achieves $8.84\%$ FDR and
$58.53\%$ power, close to using the true test null fraction.
ECOT-bi refits a pooled logistic score in each resample using L-BFGS,
tolerance $10^{-9}$, and at most 300 iterations. Its comparison therefore
includes a change in the score \citep{HuoY2026jrsssb}.

\begin{table}[!htbp]
\centering\small
\caption{Learned Gaussian score under changing class proportions with
$p=0.20$, $r=0.80$, $n=999$, $m=2000$, and $q=0.10$.
Entries are percentages (Monte Carlo standard errors) over 500 paired
resamples.  All methods except ECOT-bi use the same independently fitted
frozen score.  ECOT-bi refits its pooled classifier within every resample.}
\label{tab:learned-gaussian}
\begin{tabular}{lrr}
\toprule
Method & FDR (\%) & Power (\%) \\
\midrule
Null-reference BH & 6.61 (0.20) & 51.15 (0.88) \\
Clipped cfBH & 39.89 (0.33) & 92.81 (0.14) \\
Storey & 8.84 (0.24) & 58.53 (0.76) \\
Known-$r$ & 8.72 (0.23) & 58.91 (0.70) \\
ECOT-bi (refitted score) & 6.53 (0.20) & 50.26 (0.89) \\
\bottomrule
\end{tabular}
\end{table}
\subsection{Rank conventions and comparator implementations}
\label{sec:exp-implementation}

For methods using tie marks, each observation receives an independent
$U\sim U(0,1)$ and scores are ordered lexicographically by $(g(X),U)$.
The same marks are used across methods and pilot budgets. These marks
break score ties while leaving conformal p-values on their discrete rank
grid. Conservative null-reference BH counts score ties against selection.

For Storey, let $k\ge1$ be the number of null-reference scores, including
pilot nulls when used. Following \citet[Appendix~A.4]{BatesS2023aos}, set
\[
 \ell=\min\{k,\max\{1,\lfloor(k+1)/2\rfloor\}\},\qquad
 \lambda=\frac{\ell}{k+1},\qquad
 \widehat r_{\rm S}
 =\frac{1+\sum_{j=1}^m\ind\{p_j\ge\lambda\}}{m(1-\lambda)}.
\]
The estimate is not truncated at one. Select the largest prefix of
ordered p-values whose last index $i$ satisfies both
$p_{(i)}\le qi/(m\widehat r_{\rm S})$ and $p_{(i)}<\lambda$.
When $k=0$, select none. The known-$r$ comparator instead applies BH
at level $\min\{1,q/r\}$ using the stated null reference.

ECOT-bi fits a binary classifier with source targets in one group and
pooled source nulls and test observations in the other, with the frozen
score as an additional feature. Fitting preserves invariance to
permutations exchanging source-null and test observations. In the
Gaussian experiment, observations are sorted into a canonical order
before fitting. If no source targets are present, a constant score is
used. The fitted score supplies marked null-reference ranks, followed
by BH at level $q$. This implements the ECOT-bi construction of
\citet{HuoY2026jrsssb} with the learners specified above.

\section{Discussion and Conclusion}
\label{sec:discussion}

Calibration data affect both FDR and power, even when the scoring rule
is fixed. With matching source and test null-score distributions,
clipped cfBH controls FDR when $p\ge r$ or $r\le q$, but can exceed
the nominal level otherwise. Null-reference BH avoids this dependence
on source prevalence, although limited calibration data can prevent
it from detecting small groups of targets. The same minimum selection
size also governs sensitivity to null-score mismatch. More calibration
nulls permit smaller selections, but can make rare, high-scoring test
nulls sufficient to trigger false discoveries.

A random labeled pilot provides representative null scores and
information about test prevalence. Pooling compatible source nulls
can improve power. Without assuming compatibility, our level correction
controls FDR but yields a subset of the pilot-only selections.
Thus, the benefit of pooling depends on more than the number of
available source nulls. The guarantees require the stated independence
or conditional exchangeability assumptions and an independently fitted
score. Bounds based on null-score discrepancies also require justified
discrepancy limits, which a small labeled sample alone may not provide.
Our contribution is a quantitative analysis of these validity and
power tradeoffs for existing selection procedures. Code and supporting
experimental materials are available from the author, with experimental
methods described in Section~\ref{sec:experiments}.

\begingroup
\bibliographystyle{plainnat}
\bibliography{sjc}
\endgroup

\clearpage
\appendix
\titlespacing*{\section}{0pt}{2.4ex plus .7ex minus .2ex}{.85ex}
\titlespacing*{\subsection}{0pt}{1.65ex plus .5ex minus .2ex}{.55ex}
\section{When Grouping the Response Changes the Null Distribution}
\label{app:label-grouping}

Binary label shift on $Z=\ind\{Y\in\Ical\}$ preserves both
$P(X\mid Z=0)$ and $P(X\mid Z=1)$. Label shift in the original
response $Y$ instead preserves $P(X\mid Y)$. These are different
conditions. For any measurable set $A$,
\[
 P_\ell(X\in A\mid Z=0)
 =\int_{\Ical^c}P(X\in A\mid Y=y)\,
 P_\ell(dy\mid Z=0),
 \qquad \ell\in\{\mathrm{source},\mathrm{test}\}.
\]
Even when the conditional distribution of $X$ given $Y$ is unchanged,
the mixing distribution of $Y$ within the non-target region can
change. The resulting null covariate distribution can therefore change.

Under binary label shift, the true membership probabilities have the
same ranking when both classes have positive probability in both
populations. To see this, let
$\eta_\ell(x)=P_\ell(Z=1\mid X=x)$ and suppose $0<p,r<1$.
Bayes' formula gives
\[
 \frac{\eta_{\rm test}(x)}{1-\eta_{\rm test}(x)}
 =\frac{p(1-r)}{r(1-p)}
  \frac{\eta_{\rm source}(x)}{1-\eta_{\rm source}(x)}.
\]
The identity holds on the common mixture support, with the usual
extended-value interpretation at probabilities zero and one.
The multiplier is positive, so the ranking is preserved, although
the membership probabilities change with prevalence
\citep{LiptonZ2018icml,SugiyamaM2012book}.
This statement concerns the true probabilities and does not guarantee
the quality of a fitted score. Under the weaker common-null assumption,
the target distribution may also change, so even the true membership
ranking need not be preserved.

\section{Proofs for the Prevalence Bound and Sharpness}
\label{app:prevalence-details}
\subsection{Proof of Proposition~\ref{prop:shift-bh}}
\begin{proof}
Fix the independently fitted score. Conditional on the source and test
membership labels, Lemma~\ref{lem:conditional-bh} bounds the FDR of
null-reference BH at level $q$ by $q m_0/m$.
Since $\E[m_0/m]=r$, averaging gives $\FDR_{\rm null}\le qr$.

By \eqref{eq:shift-rank-identity}, clipped cfBH is null-reference BH
at level $\alpha_K=\min\{1,q(n+1)/(K+1)\}$.
This level is fixed conditional on the source labels.
The same lemma bounds conditional FDR by $\alpha_Km_0/m$.
The source and test labels are independent, so
\[
 \FDR_{\rm clip}\le r\E[\alpha_K]
 \le \min\{r,qr(n+1)\E[(K+1)^{-1}]\}.
\]
For $K\sim\mathrm{Bin}(n,p)$ and $p>0$,
\[
 \E[(K+1)^{-1}]
 =\int_0^1\E[t^K]\,dt
 =\int_0^1(1-p+pt)^n\,dt
 =\frac{1-(1-p)^{n+1}}{(n+1)p}.
\]
This proves the finite-sample bounds.

As $n\to\infty$, $(K+1)/(n+1)\to p$ in probability. Thus
$\alpha_K\to\min\{1,q/p\}$ in probability, and boundedness gives
convergence of its expectation. The resulting upper bound is
independent of $m$, so it holds along every sequence $m=m_n\ge1$.
If $p<q$, the event
$D_0=\{K+1\le q(n+1)\}$ has probability tending to one.
On this event every clipped p-value is at most $q$, and BH selects
the entire pool. Independence of $D_0$ and the test labels gives
$r\Prob(D_0)\le\FDR_{\rm clip}\le r$, proving convergence to $r$.
\end{proof}

\subsection{Proof of Corollary~\ref{cor:safe}}
\begin{proof}
If $p\ge r$, the second bound in \eqref{eq:shift-bh-bound} is at most
$qr/p\le q$. If $r\le q$, its first term is at most $r\le q$.
\end{proof}

\subsection{Exact FDR for the Indicator Score}

Let $g(x)=\ind\{x\in A\}$, and write
$a=P_0(A)$, $t=P_1^{\rm test}(A)$, and $d=ra+(1-r)t>0$.
Let $C$ count the source calibration nulls in $A$, and let $N_A$
count the test candidates in $A$. Then
$C\mid K\sim\mathrm{Bin}(K,a)$ and, independently of $(K,C)$,
$N_A\sim\mathrm{Bin}(m,d)$.
The clipped p-values are $(C+1)/(n+1)$ inside $A$ and
$(K+1)/(n+1)$ outside $A$.

\begin{proposition}[Exact FDR for the indicator score]
\label{prop:shift-indicator}
Under Assumption~\ref{ass:shift}, let
$D_0=\{K+1\le q(n+1)\}$. Then
\begin{equation}
\label{eq:shift-indicator-exact}
 \FDR_{\rm clip}
 =r\Prob(D_0)+\frac{ra}{d}
 \Prob\left(D_0^c,\ \frac{C+1}{n+1}\le\frac{qN_A}{m}\right).
\end{equation}
If $p>q$, $n,m\to\infty$, and $pa<qd$, FDR tends to $ra/d$
and power tends to $t$. If $p>q$ and $pa>qd$, both tend to zero.
\end{proposition}

\begin{proof}
On $D_0$, BH selects the entire pool, whose expected null fraction
is $r$. On $D_0^c$, every p-value outside $A$ exceeds $q$ and
cannot be selected. The p-values inside $A$ are equal, so BH
selects all $N_A$ candidates in $A$ exactly when
$(C+1)/(n+1)\le qN_A/m$. Otherwise, it selects none.
This event is impossible when $N_A=0$.

Conditional on all indicators $\ind\{X_j^{\rm t}\in A\}$, each
candidate in $A$ has null probability $ra/d$.
The selection event depends only on these indicators and the
independent calibration data. Its expected FDP is therefore $ra/d$,
which proves \eqref{eq:shift-indicator-exact}.

When $p>q$, $\Prob(D_0)\to0$. Also,
$(C+1)/(n+1)\to pa$ and $N_A/m\to d$ in probability.
Thus the probability of selecting exactly $A$ tends to one if
$pa<qd$, and the probability of any selection tends to zero if
$pa>qd$. In the first case, the fraction of test targets lying
in $A$ converges to $t$, because the test target proportion
converges to $1-r>0$. Since FDP and the fraction of targets
selected lie in $[0,1]$, these probability limits also give the
stated expectation limits.
\end{proof}

\subsection{Proof of Corollary~\ref{cor:shift-sharp}}
\begin{proof}
Proposition~\ref{prop:shift-bh} gives the upper bound.
For $p>q$, choose $a\in(0,1)$ small enough that
$t_0=a(p/q-r)/(1-r)<1$, and set $t=t_0+\delta$ with
$0<\delta\le1-t_0$. Then $pa<qd$, and the indicator-score
construction gives
\[
 \lim_{n,m\to\infty}\FDR_{\rm clip}
 =\frac{ra}{d}
 =\frac{ra}{pa/q+(1-r)\delta}
 \longrightarrow\frac{qr}{p}
 \qquad\text{as }\delta\downarrow0.
\]
Each $\delta$ specifies fixed distributions before either sample
size grows. Choosing it sufficiently small gives the claim for
any $\eta>0$. If $p<q$, Proposition~\ref{prop:shift-bh}
already gives the limit $r$, including for a two-valued score.

It remains to consider $p=q$. Choose $0<a<t<1$, so that
$pa<qd$. For $K\sim\mathrm{Bin}(n,q)$, the central limit theorem
gives
\[
 \Prob(D_0)
 =\Prob\{K-nq\le q-1\}\longrightarrow\frac12.
\]
The strict inequality $pa<qd$ also gives
\[
 \Prob\left\{\frac{C+1}{n+1}>\frac{qN_A}{m}\right\}
 \longrightarrow0.
\]
Consequently, \eqref{eq:shift-indicator-exact} converges to
$(r+ra/d)/2$. Keeping $a$ fixed and taking $t\downarrow a$
makes this limit approach $r$. Again, $t$ is chosen before
$n,m\to\infty$.
\end{proof}

\section{Power Bounds with Finite Calibration}
\label{app:resolution-details}
\subsection{Proof of Proposition~\ref{prop:resolution}}
\begin{proof}
Fix the training data, $K=k$, and the test membership labels.
We use the minimum-p-value argument of
\citet{MaryD2022ejs}, retaining the exact minimum
selection size and the test null fraction.
Every selected p-value is at least $1/(k+1)$ and at most $qR/m$.
Thus $R>0$ implies
$R\ge L=\lceil m/[q(k+1)]\rceil$.
If $L>m$, no selection is possible. If $m_1=0$, no target
can be selected. Both cases have zero power.

Suppose $0<m_1<L\le m$. On $\{R>0\}$,
\[
 \FDP=1-\frac{T}{R}
 \ge1-\frac{m_1}{R}
 \ge1-\frac{m_1}{L}.
\]
Lemma~\ref{lem:conditional-bh}, averaged over source label
configurations with $K=k$, gives
\[
 \left(1-\frac{m_1}{L}\right)
 \Prob(R>0\mid K=k,\mathbf Z)
 \le \E[\FDP\mid K=k,\mathbf Z]
 \le q\frac{m_0}{m}.
\]
Also, $T/(1\vee m_1)\le\ind\{R>0\}$.
These inequalities prove \eqref{eq:resolution-finite}.
When $L\le m_1$, the lower bound on FDP is nonpositive,
so this argument gives no improvement on the power bound one.
\end{proof}

\subsection{The Large-Pool Bound}

We justify \eqref{eq:resolution-limit} and \eqref{eq:resolution-unit}.
Keep $k$ fixed, and write $M_1$ for the random test target count.
If $q(k+1)<1$, then $L>m$ and power is zero for every $m$.
Otherwise,
\[
 \frac{L}{m}\longrightarrow\frac1{q(k+1)},
 \qquad
 \frac{M_1}{m}\longrightarrow1-r
 \quad\text{in probability}.
\]
Under $q(k+1)(1-r)<1$, the event $0<M_1<L$ has probability
tending to one. On this event, \eqref{eq:resolution-finite}
gives a power bound converging in probability to
\[
 \min\left\{1,\frac{qr}{1-q(k+1)(1-r)}\right\}.
\]
Use the trivial bound one on the complementary event.
The resulting random upper bound lies in $[0,1]$, so its
expectation has the same limit. Averaging conditional power
therefore proves \eqref{eq:resolution-limit}.
If
$k+1<(1-qr)/[q(1-r)]$, then
$q(k+1)(1-r)<1-qr<1$, and this limiting upper bound is
strictly below one. This proves the necessary condition
\eqref{eq:resolution-unit}.

\subsection{Positive Power Below the Unit-Power Threshold}
\label{app:resolution-example}

Take $q=0.10$, $r=0.80$, and $k=44$. Let null scores be uniform
on $[0,1]$ and every target score equal $2$.
The score perfectly separates targets from nulls.
Write $M$ for the largest calibration-null score.
Conditional on the calibration sample, the fraction of test
candidates with the minimum p-value $1/45$ converges to
$0.2+0.8(1-M)$.
If $M<35/36$, this limit exceeds $1/(45q)$.
With probability tending to one, the number of minimum-p-value
candidates therefore satisfies the BH threshold, and every target
is selected. Since $\Prob(M<x)=x^{44}$ for $0<x<1$,
boundedness of power and Fatou's lemma give
\[
 \liminf_{m\to\infty}\Power_m(44)
 \ge\Prob(M<35/36)
 =(35/36)^{44}\approx0.2895.
\]
Yet $k+1=45$ is below the value $46$ required by
\eqref{eq:resolution-unit}. Thus power need not vanish below
the threshold for power approaching one.

\section{Conditional FDR Control with Tied Scores}
\label{app:ties}
\subsection{Proof of Lemma~\ref{lem:conditional-bh}}
\begin{proof}
Fix the training data and the source and test membership labels,
with $K=k$ and $m_0$ test nulls. First take $0<\alpha<1$.
For a fixed null candidate $j$, the $k$ calibration-null scores
and its test score are i.i.d. and independent of the other
test scores.

Attach independent continuous marks to these $k+1$ observations
and sort them lexicographically, first by decreasing score and
then by decreasing mark. Marks distinguish equal scores for this
proof and do not change the conservative p-values.
Let $W$ contain the unordered marked sample and the other test
scores. Conditional on $W$, the rank $s$ occupied by candidate
$j$ is uniform on $\{1,\ldots,k+1\}$ by exchangeability.

For each possible $s$, remove the observation at rank $s$
to obtain the calibration sample. Write $p_j(s)$ for the
conservative p-value of candidate $j$. Then
$p_j(s)\ge s/(k+1)$. As $s$ increases, the sorted calibration
scores increase coordinatewise. Hence the other test p-values
are nondecreasing in $s$.
Let $\rho_s$ be the number of BH selections after replacing
$p_j(s)$ by zero. Then $\rho_s\ge1$ and $\rho_s$ is
nonincreasing in $s$.

The BH step-up rule gives the identity
\[
 j\text{ is selected}
 \quad\Longleftrightarrow\quad
 p_j(s)\le\alpha\rho_s/m,
 \qquad R=\rho_s\text{ whenever }j\text{ is selected}.
\]
Indeed, decreasing an already selected p-value to zero does not
change the ordered p-values beyond the selection threshold.
Conversely, if $p_j(s)\le\alpha\rho_s/m$, restoring its value
leaves all $\rho_s$ selected p-values below that threshold.
Monotonicity of BH then gives $R=\rho_s$.

The FDP contribution of candidate $j$ is therefore at most
\[
 \frac{\ind\{s/(k+1)\le\alpha\rho_s/m\}}{\rho_s}.
\]
The ranks satisfying this auxiliary inequality form a prefix
$\{1,\ldots,s_*\}$, possibly empty, because its left side
increases with $s$ and its right side does not.
An empty prefix contributes zero. Otherwise, averaging over
the uniform rank gives
\[
 \E\left[\frac{\ind\{j\in\Scal\}}{1\vee R}\,\middle|\,W\right]
 \le\frac1{k+1}\sum_{s=1}^{s_*}\frac1{\rho_s}
 \le\frac{s_*}{(k+1)\rho_{s_*}}
 \le\frac{\alpha}{m}.
\]
Removing different observations with the same unmarked score
leaves the same unmarked calibration multiset, so ties cause
no difficulty. Average over $W$ and sum over the $m_0$ test
nulls. This proves the bound for $0<\alpha<1$.
At $\alpha=1$, all p-values are at most one, so all candidates
are selected and FDP equals $m_0/m$.
\end{proof}

\subsection{The Exchangeability Condition Used Below}

The preceding argument uses the conformal rank and BH step-up
reasoning in \citep{BatesS2023aos,JinY2023jmlr}.
For each candidate, independence is sufficient but not necessary.
The proof only needs the reference observations and that test null
to be exchangeable conditional on the other test observations
and any variables held fixed in the analysis. Conditional on
their unordered marked sample, the test-null rank is then uniform.
This is the condition used for sampling without replacement in
Appendix~\ref{app:finite-library}.

If independent marks are used in the actual procedure, keep one
mark per observation and use the same lexicographic ordering in
every comparison. Condition also on the other test marks.
The same proof applies to the marked scores, with
$p_j(s)=s/(k+1)$.
In either implementation, the guarantee averages over the
calibration scores. It does not condition on their realized values.

\section{Proofs for Null-Score Distribution Mismatch}
\label{app:null-robustness}
\subsection{Proof of Theorem~\ref{thm:score-null-shift}}
\begin{proof}
Fix the training data and test membership labels. For each test
null $j$, let
$F_j=\ind\{j\in\Scal\}/(1\vee R)$ be its FDP contribution.
A selection requires
$1/(k+1)\le\alpha R/m$, so $R>0$ implies
$R\ge L=\lceil m/[\alpha(k+1)]\rceil$.
If $L>m$, no selection is possible.
Otherwise, $0\le F_j\le1/L$.

Consider a comparison distribution in which only the score of
candidate $j$ is changed to an independent draw from $P_0^g$.
All reference scores and other test scores retain their actual
joint distribution. In this comparison, candidate $j$ is
exchangeable with the $k$ reference nulls and independent of
the other test scores. The single-candidate argument in
Appendix~\ref{app:ties} therefore gives
\begin{equation}
\label{eq:hybrid-contribution}
 \E_j^P[F_j]\le\frac{\alpha}{m},
\end{equation}
where $\E_j^P$ averages over all scores in this comparison.
The other test nulls may still have distribution $Q_0^g$.

For a fixed score value $x$ of candidate $j$, let $h_j(x)$
be its expected FDP contribution, averaging over the unchanged
reference and other test scores. Then $0\le h_j(x)\le1/L$,
and its comparison and actual expected contributions are
$\int h_j\,dP_0^g$ and $\int h_j\,dQ_0^g$, respectively.
If $\Gamma_g<\infty$, then
\[
 \int h_j\,dQ_0^g
 \le\Gamma_g\int h_j\,dP_0^g
 \le\frac{\Gamma_g\alpha}{m}.
\]
Without absolute continuity, the total-variation bound for
a function taking values in $[0,1/L]$ gives
\[
 \int h_j\,dQ_0^g
 \le\int h_j\,dP_0^g+\frac{\varepsilon_g}{L}
 \le\frac{\alpha}{m}+\frac{\varepsilon_g}{L}.
\]
Summing over the $m_0$ test nulls proves both bounds.
The inequality $m/L\le\alpha(k+1)$ gives the simplified
total-variation bound.

At $\alpha=1$, all candidates are selected and FDP is $m_0/m$.
For lexicographic tie breaking, apply the same argument to scores
augmented with independent uniform marks. Adding the same
independent mark distribution preserves total variation and
the density-ratio bound.
\end{proof}

\subsection{Proof of Corollary~\ref{cor:joint-shift}}
\begin{proof}
Conditional on the source and test membership labels, clipped
cfBH is null-reference BH at the fixed level $\alpha_K$.
When $\alpha_K<1$ and $L_K\le m$, apply
Theorem~\ref{thm:score-null-shift}. When $L_K>m$, the FDR
is zero, so the same upper bounds remain valid.
When $\alpha_K=1$, the conditional FDP is exactly $m_0/m$.
There is then no total-variation penalty.
The density-ratio upper bound still applies because
$\Gamma_g\ge1$ whenever it is finite.

Independence of the source and test labels gives
$\E[m_0/m\mid K]=r$. Average over the labels and $K$,
and cap the bounds at one, to obtain
\eqref{eq:joint-lr} and \eqref{eq:joint-tv}.
Replacing $\alpha_K$ by $q$ gives the source null-reference
BH bounds, with $L_K$ replaced accordingly.
\end{proof}

\subsection{Proof of Proposition~\ref{prop:shift-breakdown}}
\begin{proof}
The point mass at $2$ is disjoint from the support of $P_0^g$,
so $\operatorname{TV}(Q_0^g,P_0^g)=\varepsilon$.
Every contaminated test null exceeds all calibration scores
and receives p-value $1/(k+1)$. Every target has p-value one
and is excluded because $q<1$.

Let $A$ count the contaminated test nulls. The independent
test sampling gives $A\sim\mathrm{Bin}(m,r\varepsilon)$.
If $A\ge L=\lceil m/[q(k+1)]\rceil$, then
\[
 p_{(A)}\le\frac1{k+1}\le\frac{qA}{m}.
\]
BH therefore selects at least $A$ candidates.
Every selected candidate is null, so FDP is one on this
event. This proves \eqref{eq:contamination-lower}.
If $r\varepsilon>1/[q(k+1)]$, then $A/m$ converges in
probability to a value strictly above $L/m$'s limit.
Thus $\Prob(A\ge L)\to1$ and $\FDR\to1$.

For $c>1/(qr)$ and all sufficiently large $k$,
$\varepsilon_k=c/(k+1)$ belongs to $[0,1]$ and satisfies
the strict inequality above. The inner limit in $m$ is
therefore one for every such $k$, proving the stated
iterated limit.
\end{proof}

\section{Proofs for Source-Null Pooling and Finite-Library Calibration}
\label{app:pooling-proof}
\subsection{Proof of Proposition~\ref{prop:unmatched-pooling}}
\begin{proof}
Fix the training data, source scores, and membership labels.
For a particular test null $j$, condition also on the other
test scores. By assumption, the $B$ pilot-null scores and
the score of $j$ are conditionally exchangeable.
Give these observations independent continuous marks and
condition on their unordered marked sample.
The rank $s$ of candidate $j$ is then uniform on
$\{1,\ldots,B+1\}$.

Write $N=K+B+1$. Removing rank $s$ gives the pilot reference.
As $s$ increases, this reference increases coordinatewise,
so the other test p-values are nondecreasing.
Let $\rho_s\ge1$ be the BH selection count with the p-value
of $j$ replaced by zero. It is nonincreasing in $s$.
The actual conservative pooled p-value satisfies
$p_j^+(s)\ge s/N$, because source exceedance counts
are nonnegative. The BH identity in Appendix~\ref{app:ties}
therefore bounds the contribution of $j$ by
$\ind\{s/N\le\alpha\rho_s/m\}/\rho_s$.

The accepted auxiliary ranks form a prefix
$\{1,\ldots,s_*\}$. If it is nonempty, their average
contribution is bounded by
\[
 \frac1{B+1}\sum_{s=1}^{s_*}\frac1{\rho_s}
 \le\frac{s_*}{(B+1)\rho_{s_*}}
 \le\frac{\alpha N}{m(B+1)}.
\]
An empty prefix contributes zero.
Average over the marked sample and the other test scores,
then sum over the $m_0$ nulls to prove
\eqref{eq:unmatched-conditional}.
At $\alpha=1$, every candidate is selected and FDP is $m_0/m$.
The same argument applies to actual lexicographic ranks if
independent source and other test marks are also held fixed.

For the random-pilot procedure, average the conditional
bound over source scores and labels. Its multiplicative
factor depends only on $(K,B)$, and
$\E[m_0/m\mid K,B]=r$.
Separating $\alpha(B)=1$ from $\alpha(B)<1$ gives
\eqref{eq:unmatched-unconditional}.
Infeasible finite-library allocations contribute zero,
as detailed in Appendix~\ref{app:finite-library}.

For sharpness, take $m=1$ with a null test candidate and
continuous, exchangeable pilot and test scores, all strictly
above the fixed source scores. The pooled p-value is
$s/N$, with $s$ uniform on $\{1,\ldots,B+1\}$.
The rejection probability is exactly
\[
 \frac{\min\{B+1,\lfloor\alpha N\rfloor\}}{B+1}.
\]
If $\alpha N$ is an integer not exceeding $B+1$,
this equals the conditional bound.
At other levels, the difference from the bound capped at one
is less than $1/(B+1)$.
\end{proof}

\subsection{Proof of Corollary~\ref{cor:pool-containment}}
\begin{proof}
Let $c=(B+1)/(K+B+1)>0$ and write $\alpha=\alpha(B)$.
Using the same pilot reference and tie convention, the pooled
numerator is the pilot-only numerator plus a nonnegative
source exceedance count. Thus $p_j^+\ge c u_j$ for every $j$.

For $\alpha>0$, let $R_+$ and $R_{\rm p}$ be the selection
counts of pooled BH at $c\alpha$ and pilot-only BH at $\alpha$.
If $R_+>0$, at least $R_+$ pilot-only p-values are at most
$\alpha R_+/m$, so $R_{\rm p}\ge R_+$.
Every pooled selection then satisfies
\[
 u_j\le p_j^+/c\le\alpha R_+/m\le\alpha R_{\rm p}/m
\]
and is a pilot-only selection. If $\alpha=0$, both rules
select nothing because all p-values are positive.
This proves containment, including $\alpha=1$.

FDR control follows separately from the pooling bound.
If $0<c\alpha<1$, Proposition~\ref{prop:unmatched-pooling}
gives conditional FDR at most
$c\alpha\,m_0/[cm]=\alpha m_0/m$.
If $c\alpha=0$, there are no selections.
If $c\alpha=1$, then $K=0$ and $\alpha=1$, so FDP equals
$m_0/m$. Thus the same conditional upper bound applies
at both endpoints.
For the add-one level and independent random pilot,
averaging gives $r\E[\alpha(B)]\le q$ by the binomial
calculation in Proposition~\ref{prop:pilot-fdr}.
\end{proof}

\subsection{Finite-Library Calibration}
\label{app:finite-library}

The next proposition formalizes the sampling without replacement
used for AqSolDB in Section~\ref{sec:exp-aqsol}.
Its guarantees follow from conditional exchangeability,
as in the finite-population setting of \citet{JinY2023jmlr}.

\begin{proposition}[Random pilots in a finite molecular library]
\label{prop:finite-library}
Fix a finite library with known null and target identities,
and a score fitted independently of the allocations below.
First choose a source subset. Independently of this choice,
generate $b$ pilot and $m$ test membership labels, mutually
independent with null probability $r\in(0,1)$.
Conditional on these labels and the source subset, assign
pilot and test identities uniformly without replacement within
each remaining class, independently between the two classes.
If either class lacks enough identities, return an empty
selection without redrawing the membership labels.
Let $B$ be the pilot null count.

Pilot-only conservative clipped cfBH at level $q$ satisfies
\[
 \FDR_{\rm pilot-only}
 \le r\E\!\left[\min\left\{1,\frac{q(b+1)}{B+1}\right\}\right]
 \le q\{1-(1-r)^{b+1}\}\le q,
 \qquad B\sim\mathrm{Bin}(b,r).
\]
If, conditional on source class counts, the source identities
are also sampled uniformly within each class and independently
between classes, the same bound applies to pooling source and
pilot nulls at the pilot add-one level. This is the $\gamma=0$
design. With arbitrary source sampling, the inflated bound in
Proposition~\ref{prop:unmatched-pooling} applies instead.
Both conclusions also hold for coherent independent
lexicographic tie-breaking marks.
\end{proposition}

\begin{proof}
Condition on the fitted score, source subset, and all generated
membership labels. Feasibility is now fixed.
On a feasible allocation, pilot and test null identities form
a uniform disjoint assignment from the remaining null library.
Conditional also on the other test observations, swapping a
pilot null and a particular test null preserves their distribution.
The argument in Appendix~\ref{app:ties} therefore gives
conditional FDR at most $\alpha(B)m_0/m$ for pilot-only
calibration, including conservative score ties.

An infeasible allocation contributes zero. Hence
\[
 \FDR_{\rm pilot-only}
 \le\E\!\left[\ind\{\mathrm{feasible}\}
                 \alpha(B)\frac{m_0}{m}\right]
 \le\E\!\left[\alpha(B)\frac{m_0}{m}\right]
 =r\E[\alpha(B)].
\]
The last equality uses the independent pilot and test labels
before any feasibility restriction.
Since $\alpha(B)\le q(b+1)/(B+1)$ and
\[
 \E[(B+1)^{-1}]
 =\frac{1-(1-r)^{b+1}}{(b+1)r},
\]
the stated pilot-only bound follows.

For uniform source sampling, condition on source class counts
and all pilot and test labels, but not on source null identities.
Within each class, selecting the source subset uniformly and
then allocating the remaining observations uniformly is
equivalent to a joint uniform allocation.
The independent allocations between classes ensure that
conditioning on test targets does not disturb this null symmetry.
Thus the pooled reference nulls and any test null remain
exchangeable conditional on the other test observations.
Appendix~\ref{app:ties} gives the same conditional FDR bound,
and the preceding feasibility and averaging argument applies.

For arbitrary source sampling, condition instead on the
chosen source subset. Pilot/test null exchangeability still
holds within the remaining library.
Proposition~\ref{prop:unmatched-pooling} gives its inflated
conditional bound on every feasible allocation.
Dropping the feasibility indicator and averaging yields
\eqref{eq:unmatched-unconditional}, because the generated
test labels are independent of $(K,B)$.
\end{proof}

The AqSolDB experiment uses the first $b=100$ or $500$ observations
of a fixed 500-item pilot and excludes all 500 identities from testing.
The same proof applies at either fixed budget. Conditional on the
unused pilot observations, the used pilot nulls and test nulls
retain the required exchangeability. For uniform source sampling,
this also holds with source nulls included in the reference.
After dropping the feasibility indicator and averaging over the
unused pilot, the null count in the first $b$ pilot observations
is binomial and independent of the test labels. This argument does
not permit choosing $b$ after inspecting the pilot labels or screening
results.

The implementation checks feasibility and stops if a requested
allocation is impossible. No such draw occurred in the reported
replications. The proposition gives a complete procedure by
specifying an empty selection on that event. This completion
would leave all reported selections unchanged.

\end{document}